\documentclass{article}
\DeclareUnicodeCharacter{03F5}{\ensuremath{\varepsilon}}
\usepackage{arxiv}
\usepackage[utf8]{inputenc} 
\usepackage[T1]{fontenc}    
\usepackage{hyperref}       
\usepackage{url,comment}            
\usepackage{booktabs}       
\usepackage{amsfonts,amsmath,amsthm,amssymb}       
\usepackage{mathrsfs} 
\usepackage{nicefrac}       
\usepackage{microtype}      
\usepackage{lipsum}		
\usepackage{graphicx}
\usepackage{natbib}
\usepackage{doi}
\usepackage[skins]{tcolorbox}
\usepackage{tikz}
\usetikzlibrary{arrows.meta, positioning, shapes.geometric, decorations.pathreplacing, calc}
\usepackage{adjustbox,subcaption}
\usepackage{mathtools}
\usepackage[ruled,vlined,noend]{algorithm2e}
\SetKwInput{KwInit}{Init}
\SetKwInput{KwIn}{Input}
\SetKwInput{KwOut}{Output}
\SetKwProg{Fn}{Function}{}{}
\SetKw{Break}{break}

\newtheorem{definition}{Definition}

\newtheorem{theorem}{Theorem}
\newtheorem{lemma}{Lemma}
\newtheorem{proposition}{Proposition}
\newtheorem{corollary}{Corollary}
\newtheorem{principle}{Principle}
\newtheorem{remark}{Remark}
\newtheorem{example}{Example}

\title{What is Intelligence? A Cycle Closure Perspective}

\author{ \href{https://orcid.org/0000-0003-2067-2763}{\includegraphics[scale=0.06]{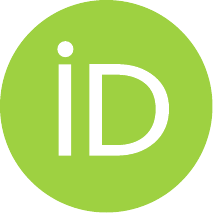}\hspace{1mm}Xin Li}\thanks{This work was partially supported by NSF IIS-2401748 and BCS-2401398. The author has used ChatGPT and Gemini models to assist the development of theoretical ideas and visual illustrations presented in this paper.} \\
	Department of Computer Science\\
	University at Albany\\
	Albany, NY 12222 \\
	\texttt{xli48@albany.edu} 
}

\renewcommand{\shorttitle}{\textit{arXiv} Template}

\hypersetup{
pdftitle={A template for the arxiv style},
pdfsubject={q-bio.NC, q-bio.QM},
pdfauthor={David S.~Hippocampus, Elias D.~Striatum},
pdfkeywords={First keyword, Second keyword, More},
}

\begin{document}
\maketitle

\begin{abstract}
What is intelligence? We argue for a structural–dynamical account rooted in a
topological closure law: \emph{the boundary of a boundary vanishes}
($\partial^2=0$). This principle forces transient fragments to cancel while
closed cycles persist as invariants, yielding the cascade
$\partial^2\!=\!0 \Rightarrow \text{cycles (invariants)} \Rightarrow \text{memory}
\Rightarrow \text{prediction (intelligence)}$. Prediction requires invariance:
only order-invariant cycles can stabilize the predictive substrate. This
motivates the \textbf{Structure-before-Specificity (SbS)} principle, where
persistent structures ($\Phi$) must stabilize before contextual specificities
($\Psi$) can be meaningfully interpreted, and is formalized by the
\textbf{Context-Content Uncertainty Principle (CCUP)}, which casts cognition as
dynamic alignment that minimizes the joint uncertainty $H(\Phi,\Psi)$. We show
that \textbf{Memory–Amortized Inference (MAI)} is the computational mechanism
that implements SbS\,$\rightarrow$\,CCUP through dual bootstrapping:
\emph{temporal} bootstrapping consolidates episodic specifics into reusable
latent trajectories, while \emph{spatial} bootstrapping reuses these invariants
across latent manifolds. This framework explains why \emph{semantics precedes
syntax}: stable cycles anchor meaning, and symbolic syntax emerges only after
semantic invariants are in place.  
In an evolutionary perspective, the same closure law unifies the
trajectory of natural intelligence: from primitive memory traces in microbes, to
cyclic sensorimotor patterns in bilaterians, to semantic generalization in
mammals, culminating in human symbolic abstraction by natural language. In
sum, intelligence arises from the progressive collapse of
specificity into structure, grounded in the closure-induced emergence of
invariants.
\end{abstract}


\section{Introduction}
\label{sec:1}

What is intelligence? Despite decades of research in cognitive science, 
neuroscience, and artificial intelligence, a unified account remains elusive \cite{hawkins2021thousand}. 
Competing definitions emphasize information processing, problem solving, 
adaptation, or symbol manipulation. In this paper, we argue for a structural 
and dynamical definition rooted in physics itself. John Archibald Wheeler, in 
his “It from Bit” vision \cite{wheeler2018information}, pointed to the fundamental clue that \emph{the 
boundary of a boundary vanishes} ($\partial^2=0$). This closure law ensures 
that transient fragments cancel while only closed cycles persist as 
invariants. From this topological root, a cascade follows: cycles provide the 
substrate of memory, memory enables prediction, and prediction defines 
intelligence. Thus, the origin of intelligence is not arbitrary or 
anthropocentric, but grounded in the same closure principle that underlies the 
physical fabric of reality:
$\partial^2 = 0 \;\;\Rightarrow\;\; \text{Cycles (invariants)} \;\;\Rightarrow\;\; 
\text{Memory} \;\;\Rightarrow\;\; \text{Prediction (intelligence)}$.

\paragraph{First Principle: Prediction Requires Invariance.}
The first step in this chain is that prediction is impossible without 
invariance. A system that treats every episode as novel cannot generalize; 
only by identifying stable patterns, invariants that persist across 
perturbations and contexts, can it reuse the past to anticipate the future \cite{zhu2021understanding,deng2022strong}. 
Cycles, guaranteed by the closure law $\partial^2=0$, supply precisely such 
invariants: they collapse order-dependent variability into robust structures. 
Biological and artificial systems alike rely on these cycles, from cyclic 
attractors in locomotion to recurrent embeddings in machine learning \cite{spalla2021continuous}. In this 
sense, invariance is the substrate of memory, and memory the substrate of 
prediction.
If invariance is the prerequisite for prediction, then the form this invariance 
must take is \emph{order invariance}. Forecasting cannot depend on the exact 
sequence in which sensory fragments arrive, since permutations and local 
deformations of order carry no stable predictive value. Instead, predictive 
stability arises when order-dependent variations cancel out unless they close 
into a cycle. In algebraic terms, the boundary of a boundary vanishes 
($\partial^2 = 0$), so only closed trajectories or nontrivial cycles persist as order-invariant carriers of memory \cite{khona2022attractor}. 
Open chains, by contrast, dissolve as transient boundaries. Cycle closure thus 
implements the filter that reduces entropy: among the many possible orderings, 
only those that stabilize into loops survive as predictive substrates. 

\paragraph{Ordering Requirement: Structure-before-Specificity (SbS).}
How are invariants extracted? We propose the 
SbS principle: low-entropy, persistent 
structures ($\Phi$) must be established before high-entropy, context-bound 
specificities ($\Psi$) can be meaningfully interpreted. In development, this is 
manifest in central pattern generators that scaffold motor control before fine 
movements are learned \cite{marder2001central}. In evolution, cyclic invariants such as navigation loops 
emerge before symbolic reasoning \cite{hauser2014mystery}. SbS explains why generalization precedes 
specialization: the backbone must stabilize before details can be elaborated.
The persistency of memory and meaning arises from the closure property of
boundaries: $\partial^2=0$ ensures that open fragments cancel while only closed
cycles endure. This mechanism turns transient episodes into stable carriers of
information. A trajectory fragment that fails to close is erased at the next
boundary operation, but one that completes a loop becomes persistent in
homology. In cognitive terms, persistency means that fleeting variations in
order or context do not destabilize memory; only those experiences that form
closed cycles survive as invariants. These invariants are precisely what can be
recalled, recombined, and communicated. 

\paragraph{Alignment Law: Context-Content Uncertainty Principle (CCUP).} The CCUP states that cognition minimizes joint uncertainty by dynamically
aligning high-entropy contextual scaffolds ($\Psi$) with low-entropy content
cycles ($\Phi$) \cite{friston2006free}. Context without content is noise; content without context is
rigid. Alignment ensures that variability is funneled into stable invariants,
yielding both robustness and adaptability. CCUP operationalizes
the SbS principle as a law of dynamic balance: $\Phi$ anchors prediction,
while $\Psi$ supplies flexibility, and their mutual alignment reduces
uncertainty to enable coherent intelligence \cite{tishby2000information}.
Under the CCUP framework, cognition is cast as the minimization of joint 
uncertainty $H(\Phi,\Psi)$, achieved through dynamic alignment between context 
and content. This principle explains both robustness (cycles that persist under 
perturbation) and flexibility (contexts that reconfigure invariants across 
situations).

\paragraph{Computational Mechanism: Memory-Amortized Inference as Implementation.}
How can SbS and CCUP be realized computationally? 
\textbf{Memory-Amortized Inference (MAI)} provides the mechanism. 
Through \emph{temporal bootstrapping}, MAI consolidates specific episodes into 
reusable latent trajectories, reducing future inference cost. Through 
\emph{spatial bootstrapping}, MAI generalizes invariants across latent 
manifolds, enabling cross-modal reuse. Together, these bootstrapping processes 
progressively collapse specificity ($\Psi$) into structure ($\Phi$), enacting 
the SbS principle in practice \cite{lee2025memory}.
The MAI framework also clarifies the relationship between semantics and syntax. 
Semantics corresponds to the stable invariants ($\Phi$) extracted by MAI; syntax 
arises only once a sufficiently rich library of cycles exists to be compressed 
and recombined. Therefore, \emph{semantics precedes syntax}, both phylogenetically 
(in evolution) and ontogenetically (in development). Human language is the 
culmination of this trajectory: symbols emerge only after semantic cycles have 
stabilized across perception, action, and memory \cite{hauser2014mystery}.

\paragraph{Evolutionary Continuity: Bootstrapping is All You Need.}
This account is consistent with the evolutionary trajectory of intelligence
\cite{bennett2023brief} (Table~\ref{tab:duality}). Primitive
organisms encode predictive traces in chemical gradients, a rudimentary form of
memory \cite{kamino2016rescaling}. With motile animals, invariants emerge through perception–action
loops \cite{fuster2004upper}, such as navigation cycles and locomotor gaits. Vertebrates add a new
layer: hippocampal–cortical dynamics align context and content, implementing
the first form of uncertainty minimization across scales \cite{buzsaki1996hippocampo}. Mammals consolidate
these dynamics through sleep and replay \cite{zhang2025replay}, a clear instantiation of
\textbf{temporal bootstrapping}: episodic specifics are compressed into durable
cycles that can be reused for future prediction. Cortical expansion further
enables \textbf{spatial bootstrapping} \cite{mountcastle1997columnar}: invariants discovered in one modality
(e.g., visual trajectories \cite{dicarlo2012does}) are redeployed in another (e.g., motor planning \cite{todorov2002optimal}),
promoting cross-modal generalization. In humans, a novel dimension emerges:
\textbf{social bootstrapping} \cite{barnes1983social}, in which cycles of
perception and action extend into cycles of communication and shared inference.
\emph{Symmetry breaking} among social agents, differences in perspective, intention, and context, creates the need to align otherwise idiosyncratic self-models. 
Language then emerges as \emph{social bootstrapping}: shared symbols compress and transmit these stabilized invariants so that multiple agents can glue their partial views into a common structure, collapsing divergent specificities into intersubjective cycles \cite{hauser2014mystery}. In this view, language builds upon, and does not create, conscious integration; it externalizes and coordinates pre-existing invariant cycles across agents. Thus, the evolutionary trajectory runs from individual closure-induced unity to socially shared global sections, where temporal, spatial, and social bootstraps converge to yield the felt unity of experience.

\begin{table}[h!]
\centering
\renewcommand{\arraystretch}{1.3}
\begin{tabular}{p{0.38\textwidth} | p{0.55\textwidth}}
\hline
\textbf{Flow of Intelligence (Conceptual)} & \textbf{Evolutionary Stage (Natural History)} \\ \hline


\textbf{Prediction requires invariance} 
& Organisms generalize across noisy inputs (robust chemotaxis, adaptive phototaxis). Without invariance, learning fails. \\ \hline \hline

\textbf{Invariance generates cycles} 
& Motile bilaterians: sensorimotor loops (locomotion, foraging) encode invariants through cyclic trajectories. \\ 

\textbf{Structure-before-Specificity principle} 
& Central pattern generators and motor primitives form stable backbones; fine-grained specificity appears later. \\ \hline \hline

\textbf{Cycle Closure / Persistence Principle} 
& Robust cyclic attractors emerge: insect gaits, vertebrate locomotion, hippocampal replay of navigation loops. \\ 

\textbf{Context-Content Uncertainty Principle (CCUP)-based dynamic alignment} 
& Vertebrates: hippocampus–entorhinal alignment; context–content separation enables flexible planning and social coordination. \\ \hline \hline

\textbf{Memory-Amortized Inference (MAI)} 
& Sleep consolidation and replay amortize single episodes into durable invariants; mammalian memory systems refine over time. \\ 

\textbf{Temporal/spatial/social bootstrapping} 
& Cortical expansion reuses cycle-structured invariants across modalities (vision→planning, audition→speech–motor). \\ \hline \hline


\textbf{Semantics abstracted into syntax} 
& Humans: stable semantic cycles scaffold symbolic language and culture; syntax emerges as compressed recombination of cycles. \\ \hline

\end{tabular}
\caption{Alignment of the conceptual flow of intelligence with evolutionary stages in natural history.}
\vspace{-0.2in}
\label{tab:duality}
\end{table}

In this paper, we therefore propose a unifying account of intelligence that 
traces its origin to topological closure. 
(i) Intelligence is rooted in the closure principle of physics: the boundary of 
a boundary vanishes ($\partial^2=0$), ensuring that transient fragments cancel 
while only closed cycles persist. 
(ii) Prediction requires invariance, because only order-invariant cycles can 
stabilize memory and enable generalization beyond rote storage. 
(iii) Such invariants emerge through the \textbf{Structure-before-Specificity 
(SbS)} principle and are dynamically maintained by the 
\textbf{Context-Content Uncertainty Principle (CCUP)}, which jointly enforce 
that persistent cycles ($\Phi$) form the backbone while transient scaffolds 
($\Psi$) supply adaptive variation. 
(iv) \textbf{Memory-Amortized Inference (MAI)} implements these principles 
computationally via dual bootstrapping: temporally, by consolidating episodic 
fragments into stable loops, and spatially, by reusing cycles across latent 
domains. 
(v) Once cycles stabilize, semantics can be abstracted into syntax: the 
dot-cycle dichotomy provides the structural bridge, where dots (fragments) 
remain ephemeral but cycles persist as carriers of meaning, memory, and 
communication.
By explicitly grounding intelligence in the \emph{topological law of cycle 
closure} and the persistence of invariants, this framework bridges cognitive 
science, machine learning, and evolutionary biology. It offers both a 
\emph{principle} (closure, SbS, CCUP) and a \emph{mechanism} (MAI 
bootstrapping) for understanding how coherence, generalization, and symbolic 
abstraction emerge from the progressive collapse of specificity into structure.


The rest of the paper is organized as follows. In Sec. \ref{sec:2}, we lay out the theoretical foundation for probing into the physical origin of intelligence in the spirit of John Wheeler. Inspired by the It-from-Bit dictum, we propose that $\partial^2 = 0$ $\Rightarrow$ topological invariance as the root principle and derive prediction-requires-invariance as the first principle. In Sec. \ref{sec:3}, we present the structure-before-specificity (SbS) principle as the consequence of cycle generation by invariance. In Sec. \ref{sec:4}, we study the dynamic alignment between context and content via cycle closure under the context-content uncertainty principle (CCUP). In Secs. \ref{sec:5} and \ref{sec:6}, we propose MAI as the computational implementation of SbS/CCUP and bootstrapping as the unified meta-strategy driving the evolution of intelligence in nature.

\section{Foundation: The Physical Origin of Intelligence}
\label{sec:2}

\subsection{Dot-Cycle Dichotomy and Topological Closure}

\paragraph{Boundary of a boundary vanishes as root principle.}
The deepest premise of our framework is a closure law from physics, $\partial^2=0$, which Wheeler flagged as a first clue in his \emph{It-from-Bit} program \cite{wheeler2018information}.
At the base of our account lies a physical closure law: the boundary of a boundary
vanishes ($\partial^2 = 0$). This identity, central in algebraic topology and echoed
in Wheeler’s \emph{It-from-Bit} dictum, enforces that transient fragments cancel
while only closed cycles persist as invariants. By organizing information into cycles
rather than open chains, closure makes invariants possible; invariants are the
substrate of memory, and memory is the substrate of prediction. Hence intelligence is
not an arbitrary construct but a corollary of a fundamental constraint on
information:
$\partial^2 = 0 \;\;\Rightarrow\;\; \text{cycles} \;\;\Rightarrow\;\; 
\text{memory} \;\;\Rightarrow\;\; \text{prediction}$.
The root of intelligence is thus a physics-level requirement: predictive substrates
must be closed under boundary operations.

\begin{figure}[t]
\centering
\resizebox{\textwidth}{!}{
\begin{tikzpicture}[
  >=Latex,
  scale=1.0,
  every node/.style={font=\small},
  panel/.style={rounded corners=6pt, draw=gray!60, line width=0.8pt, minimum width=5.6cm, minimum height=4.6cm},
  title/.style={font=\small\bfseries, text=gray!70},
  chainline/.style={line width=1.2pt, gray!70},
  cycleline/.style={line width=1.4pt, blue!70},
  endpoint/.style={circle, fill=gray!70, inner sep=1.8pt},
  arrowlab/.style={midway, above, sloped, fill=white, inner sep=1pt}
]

\node[panel] (LeftBox)  at (-6.0,0) {};
\node[panel] (MidBox)   at (  0.0,0) {};
\node[panel] (RightBox) at (  6.0,0) {};

\node[title] at ([yshift=2.5cm]LeftBox)  {Open chain $\Rightarrow$ dot in $H_0$};
\node[title] at ([yshift=2.5cm]MidBox)   {$\partial^2=0$ (boundary of boundary vanishes)};
\node[title] at ([yshift=2.5cm]RightBox) {Closed cycle $\Rightarrow$ class in $H_1$};

\draw[chainline] ($(LeftBox)+(-2.0,0.6)$) -- ($(LeftBox)+(-0.6,1.2)$) --
                 ($(LeftBox)+( 0.6,0.1)$) -- ($(LeftBox)+( 2.0,0.9)$);

\node[endpoint] (LStart) at ($(LeftBox)+(-2.0,0.6)$) {};
\node[endpoint] (LEnd)   at ($(LeftBox)+( 2.0,0.9)$) {};

\node at ($(LeftBox)+(0,1.65)$) {$\sigma \in C_1(\mathcal{Z}),\ \ \partial\sigma = \text{(start)} - \text{(end)} \neq 0$};

\draw[->, gray!70, line width=0.9pt] ($(LeftBox)+(0.1,-0.1)$) -- ($(LeftBox)+(0.1,-1.0)$)
  node[arrowlab] {\scriptsize collapse / forgetting};
\node[circle, fill=gray!70, inner sep=2pt] at ($(LeftBox)+(0,-1.4)$) {};
\node at ($(LeftBox)+(0,-1.8)$) {$H_0(\mathcal{Z})$ (dot)};

\node at ($(MidBox)+(0,1.45)$) {$\sigma$};
\draw[chainline] ($(MidBox)+(-1.7,1.0)$) -- ($(MidBox)+(1.7,1.0)$);

\draw[->] ($(MidBox)+(0,0.8)$) -- ($(MidBox)+(0,0.2)$) node[arrowlab] {\scriptsize $\partial$};
\node[endpoint] at ($(MidBox)+(-1.7,0.0)$) {};
\node[endpoint] at ($(MidBox)+( 1.7,0.0)$) {};
\node at ($(MidBox)+(0,-0.35)$) {$\partial\sigma$ (pair of endpoints)};

\draw[->] ($(MidBox)+(0,-0.6)$) -- ($(MidBox)+(0,-1.2)$) node[arrowlab] {\scriptsize $\partial$};
\node at ($(MidBox)+(0,-1.55)$) {$\partial(\partial\sigma)=0$};

\draw[cycleline] ($(RightBox)+(0,0.5)$) circle (1.2);
\node at ($(RightBox)+(0,1.85)$) {$\gamma \in C_1(\mathcal{Z}),\ \ \partial\gamma = 0$};

\draw[->, gray!70, line width=0.9pt] ($(RightBox)+(0,-0.1)$) -- ($(RightBox)+(0,-1.0)$)
  node[arrowlab] {\scriptsize persistence};
\node[blue!70] at ($(RightBox)+(0,-1.45)$) {$[\gamma] \in H_1(\mathcal{Z})$};
\node at ($(RightBox)+(0,-1.85)$) {nontrivial cycle (memory)};

\end{tikzpicture}
}
\caption{\textbf{$\partial^2=0$ enforces the dot-cycle dichotomy.} 
\emph{Left:} An open chain $\sigma$ has a nonzero boundary $\partial\sigma$ and
collapses to a dot (class in $H_0$), carrying no relational content. 
\emph{Middle:} The boundary operator squares to zero: $\partial(\partial\sigma)=0$. 
\emph{Right:} A closed chain $\gamma$ with $\partial\gamma=0$ persists as a homology
class $[\gamma]\in H_1$, i.e., a cycle that encodes order-invariant structure.}
\vspace{-0.2in}
\label{fig:dot-cycle-dichotomy}
\end{figure}

\paragraph{From topological closure to cycles.}
The identity $\partial^2=0$ means that the boundary of a boundary always 
vanishes. Algebraically, this forces a dichotomy in how information is 
organized: fragments that do not close into cycles are annihilated, while 
those that do close persist. We call this the \emph{dot-cycle dichotomy}: 1)
\emph{dots} are the ephemeral remnants of open chains: isolated events, 
transient edges, or local orderings that cannot be closed consistently. 
Because $\partial^2=0$ guarantees that every boundary fragment is canceled at 
the next level, dots never accumulate into stable carriers of meaning. They 
are cognitively analogous to momentary impressions or noise, content that 
quickly dissipates. 
2) \emph{cycles}, in contrast, are closed structures in the kernel of $\partial$. 
They survive the cancellation enforced by $\partial^2=0$ because they possess 
no further open boundaries. In homological terms, they represent equivalence 
classes $[\gamma]\in H_k(\mathcal{Z})$, immune to local perturbations. 
Cognitively, cycles correspond to stable patterns that endure across time and 
context, forming the first invariants on which memory can be built. 
In summary, topological closure is not a mere algebraic curiosity but the generative rule that separates noise (dots) from structure (cycles). This dot-cycle dichotomy 
is the first step in the emergence of intelligence: only cycles, not dots, 
become the persistent invariants that memory can retain and prediction can reuse (refer to Fig. \ref{fig:dot-cycle-dichotomy}). Formally, we have

\begin{lemma}[$\partial^2=0$ Enforces the Dot-Cycle Dichotomy]
\label{lemma:dot-cycle}
Let $C_\ast(\mathcal{Z})$ denote the chain complex of a neural state space $\mathcal{Z}$. 
The homological identity $\partial^2=0$ implies that:
1) Any open chain $\sigma \in C_1(\mathcal{Z})$ with $\partial \sigma \neq 0$ 
must collapse to a trivial 0-cycle in $H_0(\mathcal{Z})$, encoding mere connectivity 
without relational content.
2) Any closed chain $\gamma \in C_1(\mathcal{Z})$ with $\partial \gamma = 0$ 
defines a homology class $[\gamma]\in H_1(\mathcal{Z})$. If $\gamma$ is not the 
boundary of a higher-dimensional chain, it represents a nontrivial cycle that 
persists as a stable memory trace.
Thus, $\partial^2=0$ acts as a topological filter: boundaries of boundaries vanish, 
ensuring that only two outcomes are possible, collapse into trivial dots ($H_0$) 
or persistence as nontrivial cycles ($H_1$). 
\end{lemma}

The dot-cycle dichotomy in Lemma~\ref{lemma:dot-cycle} exhibits the
one-dimensional shadow of a general phenomenon: once $\partial^2=0$ holds,
open fragments cannot propagate across levels and only closed structures can
survive. For $1$-chains this yields a binary fate (collapse into trivial
$0$-cycles or persist as nontrivial $1$-cycles), mirroring the simplest
computational states. In higher dimensions, the same cancellation logic applies:
a $k$-chain contributes stably to the predictive substrate only if its boundary
vanishes (it lies in $\ker \partial_k$), and two such closed $k$-chains are
equivalent whenever they differ by a $(k\!+\!1)$-boundary (an element of
$\mathrm{im}\,\partial_{k+1}$). Thus, persistence is measured not merely by
closure, but by closure \emph{modulo} deformation through higher-dimensional
fills. This lifts the 1D dichotomy to the canonical algebraic form of
homology. The foregoing argument motivates a formal passage from intuitive persistence to algebraic structure. Once $\partial^2=0$ is assumed, the only candidates for
stable carriers are chains whose boundaries vanish, and persistence must be
assessed up to higher-dimensional fillings. This is precisely what homology
captures: closure selects $\ker \partial_k$, while indifference to deformations
through $(k\!+\!1)$-dimensional fills quotients by $\mathrm{im}\,\partial_{k+1}$.
We therefore record the following corollary.

\begin{corollary}[Topological closure yields cycles]
\label{cor:closure-cycles}
Let $(C_\bullet,\partial)$ be a chain complex with $\partial^2=0$. Then every
$k$-chain with vanishing boundary belongs to the kernel of $\partial_k$, and
the quotient $\ker \partial_k / \mathrm{im}\,\partial_{k+1}$ defines nontrivial
homology classes, i.e., cycles that are not boundaries.  
\end{corollary}

\begin{remark}[Topological vs. cycle closure]
The identity $\partial^2=0$ ensures that boundaries of chains cannot propagate
indefinitely: every open fragment is eventually canceled. What remains are
closed loops, formally the kernel of $\partial$. These cycles are the first
stable invariants of structure, surviving permutations, deformations, and local
noise. In cognition, this dot-cycle dichotomy captures 
the distinction between forgotten fragments and consolidated memories. In this sense, \emph{cycle closure} is a 
special case of \emph{topological closure}, one that grounds memory and invariance 
in biological and cognitive systems. Topological closure is the structural 
principle; cycle closure is its operational realization.
\end{remark}

\paragraph{From cycles to memory.}
Once cycles are distinguished from transient fragments, they naturally serve as
substrates of memory. In algebraic terms, a cycle is an element of 
$\ker \partial_k$ that resists annihilation by boundary cancellation; if it is
also not a boundary of a higher chain, it represents a nontrivial homology
class $[\gamma]\in H_k(\mathcal{Z})$. Such classes are stable under
perturbations: small deformations of the underlying space leave the homology
group unchanged. This \emph{stability theorem of persistent homology} implies
that cycles act as invariants, surviving noise, reparameterization, or
redundant embedding. 
Cognitively, this stability corresponds to memory. Experiences that fail to
close into cycles, like dots and fragments, dissipate and cannot be reliably
recalled. In contrast, experiences that form closed loops become encoded as
persistent traces, available for recall, recombination, and communication. A
navigation loop, for instance, can be replayed even when environmental details
change, just as a rhythmic motor cycle can be executed despite
perturbations \cite{grillner2006biological}. In both cases, the cycle furnishes a compact invariant that
outlives its specific instantiations \cite{marder2001central}. Therefore, cycles become memory not by mere storage of instances, but by persisting
as structural invariants. They collapse many distinct inputs into the same
homology class, compressing specificity into structure. This compression is the
essence of memory: it prunes away ephemeral variation, while retaining the
persistent core that can be reused for future inference.

\begin{proposition}[Nontrivial homology classes as substrates of memory]
\label{prop:cycles-to-memory}
Let $\mathcal{Z}$ be a latent space equipped with a chain complex
$(C_\bullet(\mathcal{Z}),\partial)$ and let $S:\mathcal{X}\!\to\! H_k(\mathcal{Z})$
map observables $x\!\in\!\mathcal{X}$ to their induced $k$-dimensional
homology classes $[\gamma_x]$. Suppose the observation process is (i) stable
under small perturbations (in the sense of persistent homology), and (ii)
admits a quotient $\pi:\mathcal{X}\!\to\!\mathcal{X}/\!\sim$ with
$x\!\sim\!x'$ iff $S(x)=S(x')$. Then we have:
1) (\emph{Persistence}) Each nontrivial class $[\gamma]\in H_k(\mathcal{Z})$
  is invariant under sufficiently small deformations of the data/metric, i.e.,
  $[\gamma]$ persists across a nonzero interval in the filtration.
2) (\emph{Compression}) The map $\pi$ collapses many episodic instances
  into a single class $[\gamma]$, reducing description length and empirical
  entropy on $\mathcal{X}$.
3) (\emph{Retrievability}) For any future $x^\star$ with $S(x^\star)=[\gamma]$,
  retrieval of $[\gamma]$ reconstructs a representative trajectory that is
  behaviorally adequate up to the persistence tolerance.
Hence nontrivial homology classes serve as \emph{memory substrates}: stable,
compressive, and retrievable carriers of past experience.
\end{proposition}

\paragraph{From memory to prediction.}
If cycles function as the substrate of memory, then prediction arises by
reusing these cycles to anticipate future states. Memory alone would be mere
storage if it did not afford generalization; prediction is what transforms
stored invariants into foresight \cite{vecchi2020memory}. When a new trajectory fragment arrives, it
is matched against the library of cycles in memory. If the fragment aligns with
a known cycle $[\gamma] \in H_k(\mathcal{Z})$, the continuation of $[\gamma]$
serves as a forecast for the system’s evolution. Thus, memory transforms
invariance into prediction: the past provides templates for the future.
Formally, memory collapses many order-dependent trajectories into equivalence
classes under $\sim_\Phi$, where $x \sim_\Phi x'$ iff they induce the same
cycle $[\gamma]$. Prediction then reduces to extending a new instance by
retrieving the trajectory of its equivalence class. In this way, prediction is
not recomputation from scratch but amortized reuse of persistent invariants.
Without memory, each forecast would require constructing a new model for every
case; with memory, invariants already discovered compress the hypothesis space
and guide the system toward likely continuations.
Cognitively, this principle is reflected in behaviors such as replay \cite{zhang2025replay},
anticipatory motor control \cite{todorov2002optimal}, and semantic expectation in language processing \cite{clark1996language}.
Each involves projecting a partial input onto a stored cycle and unfolding its
likely outcome. Thus prediction is not a separate faculty but a corollary of
memory grounded in cycle persistence: intelligence is realized when invariants
stored in memory are actively redeployed to anticipate the future, which leads to

\begin{proposition}[Memory enables prediction]
\label{prop:memory-to-prediction}
Let $\Phi = \{[\gamma_i]\}$ denote the set of persistent cycles extracted from
past experience. For a new trajectory fragment $x_{1:t}$, suppose there exists
$[\gamma_j] \in \Phi$ such that $x_{1:t}$ aligns with an initial segment of
$\gamma_j$ up to the persistence tolerance. Then the continuation of
$\gamma_j$ provides a valid forecast of $x_{t+1:\,T}$. Hence prediction reduces
to amortized retrieval from memory, rather than de novo computation.
\end{proposition}

\paragraph{Summary.} The logical cascade 
$\partial^2 = 0 \;\;\Rightarrow\;\; \text{cycles} \;\;\Rightarrow\;\; 
\text{memory} \;\;\Rightarrow\;\; \text{prediction}$
summarizes the generative pathway from a topological identity to the cognitive
phenomenon we call intelligence. Topological closure ($\partial^2=0$) guarantees that
transient fragments vanish, leaving only cycles; cycles persist as invariants,
providing the substrate for memory; and memory in turn enables prediction. 
To make this cascade precise, we begin by isolating the first step relevant to
intelligence proper: prediction requires invariance. Without invariance, memory
reduces to mere storage, and prediction degenerates into rote replay. With
invariance, by contrast, memory compresses episodes into stable structures that
generalize across contexts. This motivates our first principle: \emph{prediction requires invariance}.

\subsection{First Principle: Prediction Requires Invariance}


\paragraph{The fragility of non-invariance.}
Intelligence, rooted in memory-based prediction, depends on the ability to 
extract and preserve invariants. A system that encounters each situation as 
unique cannot predict, because no structure carries forward from past to future. 
It follows that invariance is the \emph{first principle} of intelligence \cite{leibo2015invariance}: prediction 
requires the persistence of patterns that remain stable under transformation, 
noise, or contextual variation.
If a representation treats every input as distinct, memory degenerates into 
mere storage (i.e., the identity crisis of over-parameterization \cite{zhang2019identity}). For example, rote memorization of past trajectories allows exact 
recall but fails to anticipate novel perturbations. Similarly, purely 
statistical correlations may fit training data but collapse under a distribution 
shift. In both cases, the absence of invariance prevents generalization: the 
system cannot infer that different inputs may map to the same outcome.

\paragraph{Invariance as cycle persistence.}
Invariants are best understood as \emph{cycles that persist} despite variation. 
A motor primitive such as a gait, or a navigation loop through an environment, 
remains robust even as sensory details fluctuate. Formally, such persistence 
corresponds to nontrivial homology classes $[\gamma] \in H_k(\mathcal{Z})$ in 
the latent space $\mathcal{Z}$ of trajectories. Topological closure ($\partial^2 = 0$) 
ensures that only cycles with vanishing boundary survive, while transient 
fragments cancel out as noise. These cycles anchor memory by collapsing many 
specific experiences into a single reusable structure.
With invariants in place, prediction becomes efficient: future states can be 
forecast by reusing persistent cycles rather than recomputing from scratch \cite{bakermans2025constructing}. 
This allows organisms to anticipate locomotor outcomes or language models to 
generalize across syntactic variants, because the same underlying structure 
recurs. Invariance provides both robustness to noise and the ability to 
generalize across contexts, the twin hallmarks of intelligent prediction. Formally, we have

\begin{principle}[Invariance as the Precondition for Generalization]
\label{prin:invariance-generalization}
Let $\mathcal{X}$ denote the space of observable trajectories and 
$\mathcal{Z}$ a latent space with homology $H_\bullet(\mathcal{Z})$. 
Suppose an agent attempts to generalize from finite samples 
$\{x_i\}_{i=1}^N \subset \mathcal{X}$ to unseen data. 
Then: \emph{generalization is possible if and only if the agent extracts 
invariants in $H_\bullet(\mathcal{Z})$ that collapse multiple instances into 
equivalence classes.} In the absence of invariants, each input remains unique 
and the agent cannot extend prediction beyond rote memorization.
\end{principle}

\noindent\textit{Proof sketch.} 
Generalization requires mapping many-to-one: different observations must be 
identified as instances of the same underlying phenomenon. Formally, this is 
the quotient $\pi:\mathcal{X} \to \mathcal{X}/\!\sim_\Phi$, where 
$x\sim_\Phi x'$ iff they induce the same cycle $[\gamma]\in H_\bullet(\mathcal{Z})$. 
If no nontrivial invariants exist, $\pi$ is the identity map and the system 
retains maximum entropy $H(\mathcal{X})$, preventing predictive compression. 
Conversely, nontrivial invariants reduce entropy by collapsing order-dependent 
variations, yielding stable predictive equivalence classes. \qed

\begin{example}[Invariance in chemotaxis and phototaxis]
Primitive organisms provide canonical cases of invariance-based prediction. 
In chemotaxis \cite{kamino2016rescaling}, bacteria move up a chemical gradient not by memorizing absolute 
concentrations (which fluctuate with noise), but by detecting \emph{relative 
changes}, an invariant under scaling of concentration. Similarly, in 
phototaxis \cite{Giometto2015PhototaxisInvariance}, single-cell algae orient toward a light source by comparing 
intensity differences across their membrane, which is invariant under global 
illumination shifts. In both cases, the prediction of a favorable direction of 
movement depends on collapsing noisy, order-dependent sensory inputs into 
stable invariants (gradient sign, intensity contrast). Without this invariance, 
the organism would be overwhelmed by noise and fail to learn or adapt.
\end{example}

\begin{corollary}[Forecasting as Entropy Reduction]
\label{cor:forecasting-entropy}
Forecasting future states requires reducing entropy by filtering 
order-dependent variations in the observation stream. Only invariant cycles 
$[\gamma]\in H_\bullet(\mathcal{Z})$ stabilize the predictive substrate: they collapse many specific sequences into a single reusable structure that 
survives perturbations.
\end{corollary}

\begin{remark}[Invariance cycles as semantic backbone]
Principle~\ref{prin:invariance-generalization} asserts that prediction without invariance degenerates into memorization. 
Corollary~\ref{cor:forecasting-entropy} refines this: forecasting is possible 
precisely when entropy is reduced through cycle detection. Invariant cycles 
act as the semantic backbone, ensuring that predictions are both 
generalizable and robust.
The principle that prediction requires invariance reshapes how we view 
intelligence. It explains why \emph{generalization precedes memorization}, why 
robustness follows from structure rather than redundancy, and why semantics must 
be grounded in persistent invariants before syntax can emerge. In the next 
section, we formalize this principle as \emph{Structure-before-Specificity 
(SbS)}, which provides the ordering rule by which invariants ($\Phi$) anchor 
prediction and specificities ($\Psi$) supply adaptive flexibility.
\end{remark}

\begin{figure}[t]
\centering
\resizebox{\textwidth}{!}{
\begin{tikzpicture}[
  scale=1.0,
  every node/.style={font=\small},
  panel/.style={rounded corners=6pt, draw=gray!60, line width=0.8pt, 
                minimum width=5.2cm, minimum height=4.6cm},
  cycletriv/.style={line width=1.4pt, blue!70},
  cyclenontriv/.style={line width=1.4pt, red!70},
  cycleorder/.style={line width=1.4pt, green!70!black, ->, >=latex},
  fillface/.style={fill=blue!12, draw=none},
  holedisc/.style={fill=white, draw=gray!65, line width=0.9pt}
]

\begin{scope}[shift={(-7.8,0)}]
  \node[panel] (LeftBox) at (0,0) {};
  \node[gray!70] at (0,2.5) {\bfseries Trivial 1-cycle};

  \fill[fillface] (0,0) circle (1.3);

  \draw[cycletriv] (0,0) circle (1.3);

  \node at (0,0) {$S$};
  \node[blue!70] at (0,-1.8) {$[\gamma]=0 \ \text{in}\ H_1$};
\end{scope}

\begin{scope}[shift={(0,0)}]
  \node[panel] (MidBox) at (0,0) {};
  \node[gray!70] at (0,2.5) {\bfseries Nontrivial 1-cycle};

  \fill[gray!06] (-2.8,-2.0) rectangle (2.8,2.0);

  \node[holedisc] (Hole) at (0,0) [circle, minimum width=2.0cm, minimum height=2.0cm] {};

  \draw[cyclenontriv] (0,0) circle (1.4);

  \node[gray!70] at (0,0) {\scriptsize hole};
  \node[red!70] at (0,-1.8) {$[\gamma]\neq 0 \ \text{in}\ H_1$};
\end{scope}

\begin{scope}[shift={(7.8,0)}]
  \node[panel] (RightBox) at (0,0) {};
  \node[gray!70] at (0,2.5) {\bfseries Order-invariant cycle};

  \coordinate (A) at (-1.5,-1.5);
  \coordinate (B) at (1.5,-1.5);
  \coordinate (C) at (1.5,1.5);
  \coordinate (D) at (-1.5,1.5);

  \draw[gray!40] (A)--(B)--(C)--(D)--cycle;

  \draw[cycleorder] (A) -- (B);
  \draw[cycleorder] (B) -- (C);
  \draw[cycleorder] (C) -- (D);
  \draw[cycleorder] (D) -- (A);

  \draw[cycleorder,dashed] (A) -- (D);
  \draw[cycleorder,dashed] (D) -- (C);
  \draw[cycleorder,dashed] (C) -- (B);
  \draw[cycleorder,dashed] (B) -- (A);

  \node[green!70!black] at (0,-1.8) 
    {$[\gamma]\ \text{independent of order}$};
\end{scope}

\end{tikzpicture}
}
\caption{\textbf{Trivial, nontrivial, and order-invariant cycles.}
\emph{Left:} A boundary of a filled region is trivial in $H_1$. 
\emph{Middle:} A loop around a hole cannot bound any 2-chain, 
so it represents a nontrivial homology class. 
\emph{Right:} Once a trajectory closes into a cycle, 
its homology class depends only on the multiset of moves, 
not their order: order permutations yield the same $H_1$ class.}
\vspace{-0.2in}
\label{fig:trivial-nontrivial-order}
\end{figure}


\paragraph{From general invariance to order invariance.}
Principle~\ref{prin:invariance-generalization} established that 
\emph{prediction requires invariance}: without stable structure, no 
generalization is possible. Yet invariance itself comes in many forms: 
translation invariance in perception \cite{dicarlo2012does}, scale invariance in physics \cite{wilson1974renormalization}, 
permutation invariance in combinatorial learning \cite{tang2021sensory}. Why does intelligence, 
conceived as memory-based prediction \cite{hawkins2021thousand}, require specifically 
\emph{order invariance}? 
The reason lies in the temporal nature of prediction. Observations arrive 
as ordered streams, but their predictive content cannot depend on the 
incidental sequence in which fragments are encountered. A predictor that is 
sensitive to local reordering would constantly change its state with each 
permutation, leading to instability \cite{lee2019set}. For predictive substrates to be 
well-defined, they must remain invariant under reparameterizations, warps, 
and permutations of input order that do not alter the underlying structure \cite{cohen2020regularizing}.

\paragraph{Cycle closure as the mechanism of order invariance.}
Unlike translation or scale invariance, which can be imposed by symmetry 
groups on static inputs, order invariance requires a topological guarantee 
that fragments do not accumulate inconsistently. This guarantee is provided 
by the closure law $\partial^2=0$ \cite{hatcher2002algebraic}. If fragments do not close into cycles, 
their endpoints remain unmatched, and permutations of order alter the 
resulting boundaries. Only when trajectories close, forming nontrivial 
cycles in homology, do the permutations cancel, leaving an order-invariant 
carrier of prediction. It follows that \emph{no cycle closure, no order invariance}.
We can therefore refine Principle~\ref{prin:invariance-generalization}:
$\text{no invariance} \;\;\Rightarrow\;\; \text{no generalization}$,
but among possible invariances, predictive stability demands
$\text{no cycle closure} \;\;\Rightarrow\;\; \text{no order invariance}$.
Together, these form the bridge: 
generalization requires invariance, and predictive stability over temporal 
streams requires invariance \emph{via cycle closure}. This transition 
explains why cycles, rather than other invariants, emerge as the primitive 
carriers of memory and foresight \cite{winfree1980geometry}.

\begin{principle}[Order Invariance via Cycle Closure]
\label{prin:order-invariance}
Let $\mathcal{X}$ be a stream of observations mapped into a latent complex 
$\mathcal{Z}$ with boundary operator $\partial$. 
Then predictive stability requires \emph{order invariance}: the predictive 
substrate must be unaffected by permutations or deformations of input order, 
except when such variations generate a nontrivial cycle. 
\end{principle}

\begin{remark}[Cycle closure leads to order invariance]
Principle~\ref{prin:order-invariance} asserts that predictive stability requires
order invariance, and that such invariance is only achieved when fragments
close into cycles. This raises a natural question: how exactly does cycle
closure enforce order invariance in practice? To answer this, we move from the
abstract statement that $\partial^2=0$ annihilates open boundaries to a
constructive formulation in terms of paths generated by local moves. In this
setting, a trajectory is built from a sequence of moves
$w=a_{i_1}\cdots a_{i_k}$, and the order of these moves may vary. If the
trajectory closes at the base state, it defines a cycle $\gamma_w$. The
critical claim is that the homology class $[\gamma_w]$ depends only on the
\emph{net multiset of moves} used (with orientations), not the specific order
in which they were composed. In other words, once cycle closure is achieved, order
dependence vanishes at the level of homology. This leads directly to the
following theorem, which formalizes cycles as the natural carriers of order
invariance.
\end{remark}

\begin{theorem}[Homological Equivalence of Permuted Cycles]
\label{thm:order-invariance}
Let $(\mathcal{Z},x_0)$ be a pointed state space with base state $x_0$. Let $\mathcal{A}=\{a_1,\dots,a_m\}$ be a set of moves inducing paths $\{\alpha_i\}$ in $\mathcal{Z}$.
Let $w=a_{i_1}\cdots a_{i_k}$ be a finite sequence of moves that yields a valid cycle $\gamma_w$ at $x_0$. Let $w'$ be any permutation of the sequence $w$. \textbf{If $w'$ also yields a valid cycle $\gamma_{w'}$ at $x_0$}, then the homology classes of the two cycles are identical:
$[\gamma_w] = [\gamma_{w'}] \in H_1(\mathcal{Z};\mathbb{Z}).$
This holds because both cycles are composed of the same multiset of path segments, and the first homology group $H_1$ is abelian.
\end{theorem}

\noindent
Theorem~\ref{thm:order-invariance} formalizes the abstract claim: once a
trajectory closes into a cycle, its predictive content no longer depends on the
precise sequence of constituent moves but only on their net combination. To see
how this principle manifests in practice, we now turn to biological examples inspired by the phylogenetic continuity hypothesis of navigation and memory \cite{buzsaki2013memory}.
Both homing trajectories and locomotor gaits demonstrate that local
reorderings and microvariations cancel as boundary terms, while the closed
cycle persists as the invariant substrate of prediction. 

\begin{example}[Homing in spatial navigation]
In spatial navigation, animals often return to a familiar location 
(``home'') after exploratory foraging \cite{o1978hippocampus}. The precise order of waypoints visited 
en route is highly variable across excursions: a rodent may take detours, 
pause, or reorder intermediate landmarks. Yet when the trajectory closes 
at the home base, the homology class of the path remains invariant. 
What matters for prediction is not the specific order of waypoints but the 
fact that the trajectory forms a closed cycle anchored at the home location \cite{foster2006reverse}. 
This closure guarantees order invariance: regardless of permutations of local 
segments, the global loop encodes a stable homing relation. Such cycles provide 
the substrate for predictive recall in the hippocampal-entorhinal system \cite{mcnaughton2006path}, 
where replayed trajectories abstract away idiosyncratic orderings while 
preserving topological connectivity of space.
\end{example}

\begin{example}[Locomotor gait cycle]
In motile animals, locomotion illustrates how cycles encode order invariance in 
the perception-action loop \cite{collins1994hard}. A quadruped gait, for instance, consists of a 
sequence of limb placements and sensory feedback signals. The exact order of 
micro-adjustments within each stride (e.g., small perturbations in muscle 
activation or ground reaction forces) varies from step to step. Yet once the 
trajectory of limb movements closes into a stable gait cycle, these local 
reorderings cancel as boundary terms, leaving the same homology class of motion \cite{grillner2009measured}. 
What persists is the \emph{cycle}, the invariant pattern of stance and swing 
phases, not the idiosyncratic order of micro-events. Prediction of future 
limb states therefore relies on the cycle structure itself: once the gait loop 
is closed, the system can anticipate the next phase regardless of noise or 
order variation in the underlying signals.
\end{example}

\begin{figure}[h]
\centering
\begin{tikzpicture}[
  >=Latex, node distance=8mm, line width=0.9pt,
  lab/.style={font=\footnotesize, align=center},
  box/.style={rounded corners, draw, inner sep=2mm},
  cyc/.style={circle, draw, minimum size=10mm},
  dot/.style={circle, fill=black, inner sep=1.2pt}
]

\node[lab] (titleL) at (-4.5,2.4) {Dot-cycle dichotomy ($\partial^2=0$)};
\node[lab] (titleR) at (4.5,2.4) {SbS: structure before specificity};

\draw[box, minimum width=8cm, minimum height=4cm] (-7,0) rectangle (-2,2);
\node[lab, anchor=west, text width=5.8cm] at (-7.5,1.7)
  {Open fragments $\to$ boundaries};

\draw (-6.3,1.2) .. controls (-5.8,1.5) and (-5.2,1.1) .. (-4.9,1.4);
\draw (-6.1,0.8) .. controls (-5.3,1.0) and (-4.7,0.6) .. (-4.3,0.9);
\draw (-5.8,0.4) .. controls (-5.2,0.2) and (-4.8,0.5) .. (-4.2,0.3);

\node[dot] at (-6.3,1.2) {};
\node[dot] at (-4.9,1.4) {};
\node[dot] at (-6.1,0.8) {};
\node[dot] at (-4.3,0.9) {};
\node[dot] at (-5.8,0.4) {};
\node[dot] at (-4.2,0.3) {};

\node[lab, text width=6cm] at (-4.5,-0.25)
  {Boundaries cancel $\Rightarrow$ no persistent carrier};

\draw[->, ultra thick] (-1.9,1.0) -- (1.9,1.0);
\node[lab] at (0,1.32) {Closure \& persistence};
\node[lab, text width=6cm] at (0,-0.6)
  {Contextual $\Psi$ is funneled into structural $\Phi$};

\draw[box, minimum width=7cm, minimum height=4cm] (2,0) rectangle (7,2);
\node[lab, anchor=west, text width=5.8cm] at (1.5,1.7)
  {Closed loops $\to$ cycles (invariants)};

\node[cyc] (c1) at (3.2,0.7) {};
\node[cyc] (c2) at (4.7,0.6) {};
\node[cyc] (c3) at (6.0,0.8) {};

\draw[->] (c1) .. controls (3.2,1.6) and (4.7,1.6) .. (c2);
\draw[->] (c2) .. controls (5.4,1.5) and (6.0,1.5) .. (c3);

\node[lab, text width=6cm] at (4.5,-0.25)
  {$[\gamma]\in H_\bullet(\mathcal{Z})$ persist (robust carriers)};

\node[box, fill=gray!10, text width=3.6cm, align=center] (psi) at (-4.5,-1.4)
  {$\Psi$: high-entropy context / scaffolds};
\node[box, fill=gray!10, text width=3.6cm, align=center] (phi) at (4.5,-1.4)
  {$\Phi$: low-entropy persistent cycles (semantics)};

\draw[->] (psi) -- node[midway, lab, text width=5.8cm]
  {SbS ordering: structure first, specificity second} (phi);

\node[lab, text width=11cm] at (0,-2.5)
  {Dot-cycle dichotomy $\Rightarrow$ only cycles survive.
   SbS: use cycles $\Phi$ as backbone; treat dots $\Psi$ as scaffolds.};

\end{tikzpicture}

\caption{SbS as a corollary of the dot-cycle dichotomy. Open fragments cancel
as boundaries, whereas closed loops persist as invariants. Closure (via
$\partial^2=0$) and persistence funnel contextual variability $\Psi$ into
structural invariants $\Phi$, which anchor memory and prediction; specificity
is layered afterward.}
\label{fig:sbs-dot-cycle}
\end{figure}

\section{Ordering Requirement: Structure-before-Specificity (SbS)}
\label{sec:3}

\paragraph{Structure-before-Specificity (SbS) as a corollary of the dot-cycle dichotomy.}
If prediction requires invariance, the next question is how invariants emerge
from raw experience. The \textbf{SbS principle}
answers this from two complementary lenses: topological vs. information theoretic (refer to the top and bottom panels in Fig. \ref{fig:sbs-dot-cycle}).
The SbS principle follows directly from the
dot-cycle dichotomy implied by the closure law $\partial^2 = 0$. Algebraically,
open fragments collapse as trivial dots while closed loops persist as nontrivial
cycles. Cognitively, this enforces an ordering: persistent cycles $\Phi$ must
stabilize first, since only they survive boundary cancellation, while transient
fragments $\Psi$ can serve at most as exploratory scaffolds. SbS therefore
operationalizes the dot-cycle dichotomy in cognitive terms: \emph{structure
(cycles) before specificity (dots)}. Without this ordering, memory would be
anchored to transient features and prediction would collapse.

\paragraph{Topological view.}
Formally, let $\mathcal{Z}$ be a latent complex with boundary operator
$\partial$ and homology $H_\bullet(\mathcal{Z})$ \cite{hatcher2002algebraic}. The closure identity
$\partial^2 = 0$ induces the dot-cycle dichotomy: open fragments cancel as
boundaries, while closed chains define persistent cycles
$[\gamma]\!\in\!H_\bullet$. SbS declares that these nontrivial, persistent
cycles form the structural backbone $\Phi$, whereas transient or trivial cycles
constitute the contextual scaffolds $\Psi$. Persistence (via filtrations)
ensures robustness: classes with lifetime $>\tau$ are stable under perturbations
of size $<\tau$ \cite{edelsbrunner2008persistent}, so $\Phi$ must be formed
first to anchor memory and generalization. Only after $\Phi$ is established do
we admit $\Psi$, local deformations, parameterizations, and boundary terms that
adapt the backbone without altering its homology class. In sheaf-theoretic terms
\cite{ayzenberg2025sheaf}, SbS requires that compatible local posteriors glue to
a global section (semantic structure), while residual mismatches live on overlaps
as contextual corrections constrained to vanish under descent.

\paragraph{Information-theoretic view.} \cite{cover1999elements}
Let $X$ be observations, $Y$ the predictive target, and let $\Phi:=S(X)$ denote
low-entropy \emph{content} extracted from $X$ while $\Psi:=R(X)$ denotes
high-entropy \emph{specificity} (contextual degrees of freedom within a content
class). SbS requires that $\Phi$ be (near-)sufficient for prediction and that
$\Psi$ contributes at most a small residual:
$Y \perp X | \Phi \wedge I(Y;\Psi | \Phi) \le \varepsilon (\varepsilon \ll I(Y;\Phi))$.
Equivalently, $\Phi$ is the minimal sufficient statistic (in the minimum description length(MDL)/rate-distortion(RD) sense \cite{grunwald2007minimum})
that compresses $X$ while preserving predictive information about $Y$:
$\Phi^\star \in \arg\min_{\Phi:\, I(Y;\Phi)\ge \kappa}\ H(\Phi)
\quad\text{or}\quad
\Phi^\star \in \arg\min_{\Phi}\ \big\{ H(\Phi) + \lambda\,\mathbb{E}\,d(Y,\hat Y(\Phi)) \big\}$.
This way, SbS imposes an \emph{ordering of channels}: first extract a low-entropy,
predictively sufficient $\Phi$ (structure), then allow $\Psi$ to modulate or
refine predictions within the constraint $I(Y;\Psi | \Phi)\!\le\!\varepsilon$
(specificity). Without this ordering, learning either overfits to $\Psi$ (brittle,
non-generalizing) or becomes underfit (rigid, non-adaptive).

\paragraph{Structure as persistent cycles and specificity as contextual scaffolds.}
We identify structure ($\Phi$) with persistent cycles, invariants in the latent 
space of trajectories. These cycles survive perturbations because their 
boundaries vanish ($\partial^2=0$), yielding robust memory traces. For example, 
a locomotor gait remains stable across sensory fluctuations \cite{collins1994hard}, and a 
navigation loop remains identifiable despite path variations \cite{muller1996hippocampus}. Such persistent 
structures supply the backbone of memory and predictive power \cite{vecchi2020memory}.
In contrast, specificity ($\Psi$) corresponds to transient or trivial cycles: 
high-entropy scaffolds that may collapse but still serve a purpose. They 
provide the variability needed for exploration, adaptation, and contextual 
fine-tuning. For example, trial-and-error movements in motor learning or 
exploratory vocalizations in language acquisition do not persist as invariants, 
but they scaffold the eventual stabilization of persistent patterns \cite{uehara2019interactions}. We formalize the above intuition into the following definition and a canonical principle.

\begin{definition}[Observation, content, specificity]
Let $\mathcal{X}$ be the observation space and $\mathcal{Z}$ a latent space
with chain complex $(C_\bullet(\mathcal{Z}),\partial)$ and homology
$H_\bullet(\mathcal{Z})$. The \emph{content map} (semantics)
$S:\mathcal{X}\to H_\bullet(\mathcal{Z})$ assigns each $x\in\mathcal{X}$ a
persistent cycle class $[\gamma_x]$.
Define the \emph{content equivalence} $x\sim_\Phi x'$ iff $S(x)=S(x')$, and
let $\pi:\mathcal{X}\to \mathcal{X}/\!\sim_\Phi$ be the quotient onto
\emph{content classes} (low entropy, persistent).
Let $\Psi$ denote \emph{specificity}: context variables or fast degrees of
freedom that vary within a content class.
\end{definition}

\begin{principle}[SbS Ordering Requirement]
\label{prin:SbS-order} 
Cognition must respect the following ordering requirement, expressed in three
equivalent forms:
\begin{enumerate}
  \item \textbf{Homological form.} 
  Any admissible representation $\rho:\mathcal{X}\to\Sigma$
  used for prediction must \emph{factor through content classes}:
  $\exists\,e:\ \mathcal{X}/\!\sim_\Phi \to \Sigma \quad \text{s.t.} \quad
  \rho \;=\; e\circ \pi $.

  \item \textbf{Information-theoretic form.} 
 Let $Y$ be the predictive target. Then $\Phi:=S(X)$ is the primary predictor and $\Psi$ is residual:
$I(Y;\Psi\,|\,\Phi) \le \varepsilon \quad (\text{where } \varepsilon \ll I(Y;\Phi)) $.

  \item \textbf{Dynamical form.} 
  Along a learning trajectory $t\mapsto (\Phi_t,\Psi_t)$,
  persistent content grows monotonically while residual reliance on specificity
  shrinks:
  $\Phi_{t+1} \supseteq \Phi_t, \qquad
  \mathbb{E}[\Delta\mathcal{L}_{t+1}\,|\,\Phi_{t+1}]
  \;\le\;
  \mathbb{E}[\Delta\mathcal{L}_{t}\,|\,\Phi_{t}]$,
  with strict improvement whenever a new nontrivial cycle enters $\Phi_{t+1}$.
\end{enumerate}
\end{principle}

\paragraph{Instantiations of SbS.}
The three formulations of Principle~\ref{prin:SbS-order} can be grounded in
canonical examples across domains:

\begin{itemize}
  \item \textbf{Homological form $\to$ Spatial cognition.} 
  In the hippocampal-entorhinal system, grid cells provide a stable lattice 
  $\Phi$ (persistent cycles in $\ker \partial$) while place cells supply 
  contextual refinements $\Psi$ \cite{moser2008place}. The homological factorization condition 
  ($\rho = e\circ \pi$) corresponds to a prediction depending only on grid-based 
  equivalence classes, with place-cell variability entering only as residual 
  scaffolding.

  \item \textbf{Information-theoretic form $\to$ Motor control.}
  In locomotor circuits, central pattern generators (CPGs) generate rhythmic 
  invariants $\Phi$ that are low-entropy and sufficient for predicting the next 
  motor phase \cite{grillner2006biological}. Proprioceptive and exteroceptive corrections $\Psi$ provide 
  high-entropy specificity that refines but does not overturn the backbone. 
  Information-theoretically, $I(Y;\Phi)\gg I(Y;\Psi | \Phi)$: most predictive 
  information comes from the structure (CPG cycles), with only small residual 
  gain from specificity.

  \item \textbf{Dynamical form $\to$ Affect and cognition.}
  In the affective domain, basic emotions (fear, joy, anger) act as persistent 
  attractors $\Phi$ that monotonically stabilize across evolutionary and 
  developmental time \cite{panksepp2004affective}. Cognitive processes $\Psi$ (appraisal, reappraisal, 
  inhibition) adaptively refine these structures but shrink in residual 
  contribution as content stabilizes. Dynamically, $\Phi_{t+1}\supseteq \Phi_t$: 
  emotional backbones persist and expand, while $H(\Psi | \Phi)$ decreases as 
  cognition learns to align with structure rather than destabilize it.
\end{itemize}

\begin{corollary}[Equivalence of SbS orderings]
\label{cor:SbS-equivalence}
Under mild regularity (stable encoder, boundary insensitivity, and bounded
residual channel), the homological, information-theoretic, and dynamical forms
in Principle~\ref{prin:SbS-order} are equivalent: each implies the others.
\end{corollary}

\begin{remark}[Ordering requirement for robustness and flexibility]
SbS is an \emph{ordering requirement}: (i) representations must first collapse
observations to content classes (factor through $\pi$); (ii) content carries the
predictive load while specificity contributes at most a small residual; (iii)
learning dynamics must prioritize persistent cycles and only then fine-tune
context. Operationally, this prevents \emph{specificity-first} failure modes
(brittle memorization, overfitting \cite{zhang2018dissection}) and guarantees that semantics anchors
prediction before syntax or surface variation is exploited. 
SbS states that $\Phi$ must precede $\Psi$ in explanatory priority: structure first, specificity second. If specificity comes first, the system is flooded with high-entropy variability and cannot generalize. If structure comes first, specificity can be constrained and funneled into persistent invariants. This ordering underlies both development (pattern generators before fine motor skills \cite{thelen1994dynamic}) and evolution (navigation cycles before symbolic reasoning).
This abstract ordering has immediate functional consequences. The SbS principle explains the dual properties of intelligence: robustness and flexibility. Robustness arises from persistent cycles ($\Phi$) that anchor prediction across perturbations. Flexibility arises from contextual scaffolds ($\Psi$) that enable adaptation to novelty. Intelligence emerges from the dynamic interplay of these 
two forces \cite{bennett2023brief}, with structure ensuring stability and specificity enabling 
variation. In this sense, SbS is not only an ordering rule but also a design principle: it guarantees a substrate that is both stable enough to generalize and flexible enough to adapt.
\end{remark}
\paragraph{From robustness to representation.}
If robustness and flexibility are to be implemented computationally, they must 
be grounded in representational terms. This requires distinguishing two layers: 
\emph{semantics}, which captures structural invariants, and \emph{syntax}, which 
encodes surface-level specificity. The following definitions make this precise.

\begin{definition}[Semantics and Syntax]
Let $\mathcal{X}$ be the space of observables (inputs, trajectories).
A \emph{semantic map} $S:\mathcal{X}\to H_\bullet(\mathcal{Z})$ assigns to each 
$x\in\mathcal{X}$ a structural invariant (its persistent homology class), i.e., 
its \emph{meaning}.
A \emph{syntactic map} $\sigma:\mathcal{X}\to\Sigma$ encodes $x$ into a surface 
representation (symbols, strings, program states).
An \emph{interpretation} $I:\Sigma\to H_\bullet(\mathcal{Z})$ makes syntax 
meaningful when $I\circ \sigma \approx S$ (up to stability).
\end{definition}

Given these definitions, the ordering requirement of SbS now manifests as a 
constraint on how syntax can be meaningful. Specifically, syntax must be 
factored through semantic classes; otherwise, it risks collapsing into 
noise-sensitive encoding \cite{van2005exploring}. This also highlights the fundamental limitation of 
syntax-based computation as epitomized by the Turing machine model \cite{turing1936computable}. A Turing 
machine operates purely on symbol strings through sequential rewriting rules. 
Its power comes from exhaustive enumeration: all possible symbol configurations 
are, in principle, traversable. However, this enumerative nature makes the Turing 
model inherently fragile with respect to invariance: any small perturbation in 
symbol order alters the trajectory entirely, producing a new and unrelated 
string. In contrast, SbS requires that symbolic syntax $\sigma$ be constrained 
by, and derived from, semantic invariants $\Phi$ that are robust to order 
variation. Put differently, while the Turing paradigm equates intelligence with 
symbol manipulation, SbS insists that symbols only gain meaning when rooted in 
persistent cycles. This ordering requirement exposes a structural 
limitation of Turing-style computation \cite{sipser1996introduction}: syntax without semantics is brittle, 
exploding into combinatorial complexity, whereas syntax anchored to semantic 
classes compresses the search space and yields robust generalization.
This yields the following proposition.

\begin{proposition}[Semantics-Before-Syntax as a Consequence of SbS]
\label{prop:semantics-before-syntax}
Assume Principle~\ref{prin:SbS-order}. Then any admissible syntactic encoder 
$\sigma:\mathcal{X}\to\Sigma$ that is \emph{meaningful} 
(i.e., admits an $I$ with $I\circ\sigma \approx S$) must \emph{factor through} 
the semantic quotient $\pi$:
$\exists\, e:\mathcal{X}/\!\sim_\Phi \to \Sigma \quad \text{s.t.} \quad 
\sigma \;=\; e \circ \pi$,
and, moreover, $I\circ e$ is constant on each content class. Hence, 
\emph{semantics (structure $\Phi$) precedes syntax (specificity $\Psi$)}: 
syntax is a realization that is constrained by, and derived from, semantic 
structure.
\end{proposition}

\paragraph{Quotients and equivalence.}
To make this factoring explicit, we must distinguish between the raw space of 
observations $\mathcal{X}$ and the equivalence classes induced by persistent 
structure. Two observations $x,x'\in\mathcal{X}$ are content-equivalent if they 
map to the same semantic invariant under $S$, i.e., they collapse to the same 
homology class. Syntax can only be meaningful if it first respects this 
partition: rather than directly encoding raw inputs, $\sigma$ must act on 
equivalence classes, with the quotient $\pi$ providing the canonical projection 
onto semantic content \cite{huh2023isometric}.

\begin{definition}[Content Equivalence]
Define $x\sim_\Phi x'$ iff $S(x)=S(x')$ in $H_\bullet(\mathcal{Z})$. 
Let $\pi:\mathcal{X}\to \mathcal{X}/\!\sim_\Phi$ be the canonical quotient onto 
\emph{content classes}.
\end{definition}

This quotient construction turns the abstract requirement of 
Proposition~\ref{prop:semantics-before-syntax} into a concrete constraint: syntactic encoders must operate \emph{through} semantic classes if they are to 
preserve meaning \cite{duneau2025towards}. Formally, the quotient map $\pi:\mathcal{X}\!\to\!\mathcal{X}/\!\sim_\Phi$
identifies all inputs with the same semantic invariant, and \emph{factoring through} means that there exists an $e$ with $\sigma=e\circ\pi$, so that $\sigma$
depends only on the content class and not on a particular representative. 
The following three examples illustrate how this factorization manifests in practice.

\begin{example}[Infant language acquisition: meanings before grammar]
Before mastering grammar, infants acquire stable semantic invariants: they
reliably recognize caregivers, objects, and common actions across viewpoints and
contexts (low-entropy content $\Phi$). This has been called semantic bootstrapping hypothesis by Pinker \cite{pinker1984semantic}. Only after these invariants stabilize do
they attach words and, later, compose sentences (surface syntax $\Psi$). In the
terms of Prop.~\ref{prop:semantics-before-syntax}, the emerging encoder
$\sigma$ factors through content classes $\pi$: different utterances that refer
to the same grounded concept map to the same meaning $S(x)$, while syntactic
variation is treated as residual within-class variability.
\end{example}

\begin{example}[Spatial cognition: grid/place structure before route syntax]
In navigation, grid cells furnish an invariant spatial scaffold (semantic
backbone $\Phi$) that persists across contexts; place cells then anchor specific
locations as contextual refinements ($\Psi$). Route descriptions and turn-by-turn
“syntax” are meaningful only insofar as they factor through the grid/place
structure \cite{moser2008place}: two differently phrased route instructions that traverse the same
homology class (loop) decode to the same meaning $S(x)$, demonstrating that
syntax is constrained by, and derived from, semantic cycles.
\end{example}

\begin{example}[Motor control and speech: rhythmic backbones before articulation]
Central pattern generators (CPGs) establish rhythmic invariants (gait or vocal
prosody) that serve as low-entropy structure $\Phi$ \cite{grillner2006biological}. Fine articulatory gestures
(phonemes, syllables, coarticulation) provide high-entropy specificity $\Psi$
layered on top. Meaningful “speech syntax” (strings of phonemes) must factor
through the prosodic/motor backbone: perturbations that preserve the underlying
rhythmic cycle leave the intended meaning unchanged, while purely syntactic
rearrangements without backbone alignment fail to stabilize \cite{barlow2006central}.
\end{example}

\paragraph{Consequences of factoring.}
Once syntax is forced to factor through semantic classes, two key consequences 
follow: generalization across syntactic variants and robustness against 
perturbations. In effect, factoring guarantees that syntax is no longer a 
free-floating, brittle enumeration of surface forms (as in Turing-style string 
manipulation \cite{sipser1996introduction}), but instead a constrained encoding tied to persistent structure. 
This shift collapses redundancy, compresses variability, and channels syntactic 
expressions through semantic invariants. The following corollaries make these 
consequences precise.

\begin{corollary}[Generalization by Equivalence Classes]
\label{cor:generalization}
If two inputs $x,x'$ satisfy $x\sim_\Phi x'$, then any meaningful syntax 
$\sigma=e\circ\pi$ yields $I(\sigma(x))=I(\sigma(x'))$ (up to the stability 
margin). Thus all syntactic variants within a content class generalize to the 
same meaning.
\end{corollary}

\paragraph{From equivalence to stability.}
Corollary~\ref{cor:generalization} treats \emph{discrete} variability: distinct
inputs that fall into the same content class ($x\sim_\Phi x'$) must yield the
same meaning under any meaningful syntax. This ensures that multiple surface 
forms are unified by the same underlying invariant. The next step is to extend 
this guarantee from discrete equivalence to \emph{continuous} variability: 
small perturbations that do not alter the persistent homology class $S(x)$. 
Here the topology of persistence provides the margin of stability \cite{cohen2005stability}: if a cycle 
remains intact under deformation, then syntax grounded in that cycle must also 
remain invariant. Corollary~\ref{cor:robustness} formalizes this, upgrading 
equivalence-based generalization to robustness under perturbation. In short, 
sameness of class implies generalization; persistence of class implies 
robustness.

\begin{corollary}[Robustness via Persistence]
\label{cor:robustness}
Let $\delta$ be a perturbation of $x$ that does not change the persistent 
homology class $S(x)$ (i.e., stays below the relevant persistence threshold).
Then $\pi(x)=\pi(x+\delta)$ and any meaningful $\sigma=e\circ\pi$ is invariant 
to $\delta$ up to stability, yielding robustness to noise and deformation.
\end{corollary}

\begin{remark}[Compositionality and Sheaf Structure]
The factoring requirement does not only yield robustness but also enables 
structured reuse. If cycles compose via concatenation or pushforward in 
$H_\bullet(\mathcal{Z})$, then content classes inherit a natural monoidal 
structure \cite{gallier2022homology}. Any meaningful syntax $\sigma=e\circ\pi$ 
that respects this composition (i.e., $e$ is a functorial encoding) yields 
\emph{compositional semantics}: local cycles serve as primitives that can be 
combined into more complex structures without loss of meaning. 
This compositionality is naturally interpreted in sheaf-theoretic terms \cite{ayzenberg2025sheaf}. 
Content classes act as \emph{local sections}, while functorial encoders $e$ 
preserve their algebraic structure under restriction and concatenation. 
Compatibility of local sections ensures that they glue into coherent global 
semantics, bridging local invariants with global expressivity. In this sense, 
SbS secures the persistence of local cycles, while CCUP enforces their dynamic 
alignment into global structure. Sheaf theory thus provides the categorical 
framework unifying these principles: meaning is preserved when syntax factors 
through semantic quotients and local invariants compose consistently into 
global closure.
\end{remark}

\paragraph{Summary.}
SbS is therefore an \emph{ordering requirement} that aligns compression with 
closure: syntax is not allowed to operate on raw variability but must first be 
anchored to persistent structure. Information-theoretically, $\Phi$ is a 
low-entropy, nearly sufficient statistic for $Y$, while $\Psi$ contributes only 
bounded residual variability. Topologically, $\Phi$ is the set of persistent 
cycles that survive boundary cancellation ($\partial^2=0$). Specificity $\Psi$ 
is then layered \emph{after} structure, ensuring robustness (from persistence 
and low entropy) and adaptability (from controlled residual information). 
SbS establishes the order of construction: structure before specificity. 
Yet order alone is not enough. Intelligence also requires that variability in 
$\Psi$ remain dynamically aligned with $\Phi$ as contexts change. This is 
captured by the \textbf{Context--Content Uncertainty Principle (CCUP)}, which 
formalizes cognition as the minimization of joint uncertainty through dynamic 
alignment. Whereas SbS enforces the representational hierarchy, CCUP enforces 
the inferential balance: $\Phi$ anchors prediction, $\Psi$ supplies flexibility, 
and their alignment reduces $H(\Phi,\Psi)$. Together, SbS and CCUP form the 
theoretical foundation for \textbf{Memory--Amortized Inference (MAI)}, which 
implements both principles through temporal and spatial bootstrapping.


\section{Alignment Law: Context-Content Uncertainty Principle}
\label{sec:4}

\paragraph{From SbS to CCUP.}
The SbS principle provides the ordering rule:
persistent structures ($\Phi$) must stabilize before contextual specificities
($\Psi$) can be interpreted. The Context-Content Uncertainty Principle (CCUP)
generalizes this insight by treating cognition as the minimization of joint
uncertainty $H(\Phi,\Psi)$ through dynamic alignment \cite{li2025CCUP}. In this view, SbS can be
read in two complementary ways: (i) as a precursor to CCUP, establishing the
ordering constraint that makes alignment possible, or (ii) as a special case of
CCUP, corresponding to the regime where uncertainty minimization enforces strict
priority of $\Phi$ over $\Psi$. Thus, SbS and CCUP are not competing principles
but nested layers: SbS provides the developmental ordering, while CCUP provides
the dynamical law.
To operationalize this picture in cognition, we adopt the
\emph{Context–Content Uncertainty Principle (CCUP)} (Fig. \ref{fig:ccup-feedback}): stable
memory traces correspond to low-entropy \emph{content variables} $\Phi$
(persistent homological cycles), while transient variability is captured by
high-entropy \emph{context variables} $\Psi$. In what follows, we show how
\emph{Memory–Amortized Inference (MAI)} implements cycle formation by holding
$\Phi$ fixed as reusable structure and adapting $\Psi$ until residual
boundaries cancel ($\partial^2=0$), thereby achieving topological closure.

\paragraph{Content variable $\Phi$ as low-entropy homology.}
Within CCUP, the content variable $\Phi$ corresponds to information that is 
both specific and stable. Mathematically, $\Phi$ is identified with 
nontrivial homology classes \cite{hatcher2002algebraic}: cycles $[\gamma] \in H_k(\mathcal{Z})$ that 
cannot be reduced to boundaries (e.g., the middle panel in Fig. \ref{fig:trivial-nontrivial-order}). Such cycles encode persistent, 
low-entropy structures because many possible trajectories or micro-states 
collapse into the same equivalence class. In neural terms, $\Phi$ reflects 
patterns of activity that recur reliably across different contexts, such as 
a learned motor primitive \cite{stroud2018motor}, a familiar spatial route \cite{mcnaughton1991dead}, or a well-established 
object representation \cite{dicarlo2012does}. By filtering away order-dependent variability, 
$\Phi$ preserves only the invariant relational structure that remains after 
symmetry breaking. This makes $\Phi$ the stable substrate of memory and the 
carrier of predictive power \cite{vecchi2020memory}: once identified, it can be recalled, reused, 
and composed into higher-order cognitive structures.

\paragraph{Context variable $\Psi$ as high-entropy scaffolding.}
In contrast, the context variable $\Psi$ captures the transient, exploratory, 
and often noisy aspects of cognition. Topologically, $\Psi$ is associated with 
trivial cycles (the left panel in Fig. \ref{fig:trivial-nontrivial-order}) or short-lived features in the persistence barcode: loops that 
quickly vanish under perturbation or deformation. These cycles act as 
\emph{scaffolding}, supporting the discovery and stabilization of $\Phi$ but 
not themselves persisting as memory. In information-theoretic terms, 
$\Psi$ is high-entropy: it reflects a large space of possibilities, many of 
which will be pruned away as the system concentrates its measure on 
low-entropy $\Phi$ structures. Biologically, $\Psi$ is implemented by 
slow, contextual rhythms (e.g.\ theta oscillations \cite{buzsaki2006rhythms}) or exploratory neural 
activity that supplies diverse scaffolds for binding \cite{treisman1980feature}. Through dynamic 
alignment and phase-resetting \cite{tass2007phase}, these high-entropy contextual structures are 
folded into persistent content loops, allowing cognition to maintain 
flexibility while ensuring stability in memory formation.

\begin{figure}[h]
\centering
\resizebox{\textwidth}{!}{
\begin{tikzpicture}[
  >=Latex, line width=0.9pt,
  box/.style={rounded corners, draw, inner sep=2mm},
  lab/.style={font=\footnotesize, align=center},
  cyc/.style={circle, draw, minimum size=9mm},
  dot/.style={circle, fill=black, inner sep=0.9pt},
  ctrl/.style={draw, rounded corners, minimum width=18mm, minimum height=9mm, fill=gray!10},
  meas/.style={draw, diamond, aspect=2, inner sep=1.6pt, fill=gray!05}
]

\draw[box, minimum width=6.8cm, minimum height=3.8cm] (-8.2,0.2) rectangle (-1.8,2.6);
\node[lab] at (-5.0,2.9) {Context $\Psi$};

\draw (-7.4,2.0) .. controls (-6.8,2.3) and (-6.1,1.8) .. (-5.7,2.1);
\node[dot] at (-7.4,2.0) {};
\node[dot] at (-5.7,2.1) {};

\draw (-7.0,1.3) .. controls (-6.2,1.5) and (-5.5,1.1) .. (-5.0,1.4);
\node[dot] at (-7.0,1.3) {};
\node[dot] at (-5.0,1.4) {};

\draw (-6.5,0.7) .. controls (-5.8,0.5) and (-5.2,0.8) .. (-4.6,0.6);
\node[dot] at (-6.5,0.7) {};
\node[dot] at (-4.6,0.6) {};

\draw[box, minimum width=6.8cm, minimum height=3.8cm] (1.8,0.2) rectangle (8.2,2.6);
\node[lab] at (5.0,2.9) {Content $\Phi$};

\node[cyc] (c1) at (3.2,1.9) {};
\node[cyc] (c2) at (4.9,1.3) {};
\node[cyc] (c3) at (6.6,1.9) {};

\draw[->, very thin] (c1) .. controls (3.2,2.4) and (4.9,2.4) .. (c2);
\draw[->, very thin] (c2) .. controls (5.5,2.6) and (6.6,2.6) .. (c3);

\node[ctrl] (enc) at (0.0,1.45) {inference};
\draw[->, thick] (-1.8,1.45) -- (enc.west);
\draw[->, thick] (enc.east) -- (1.8,1.45);

\node[meas] (meas) at (0.0,0.25) {};
\node[lab] at (0.0,-0.05) {mismatch};

\draw[->, thin] (-0.3,1.00) -- (meas.north west);
\draw[->, thin] (0.3,1.00) -- (meas.north east);

\node[ctrl] (ctrl) at (-4.0,-0.6) {controller};
\draw[->, thick] (meas.west) .. controls (-1.6,-0.6) and (-3.0,-0.6) .. (ctrl.east);
\draw[->, thick] (ctrl.west) .. controls (-8.6,-0.3) and (-9.4,1.05) .. (-8.2,1.05);

\node[lab] at (0.0,3.15) {$\partial^2=0$ enforces closure};
\node[lab] at (0.0,2.80) {non-closing fragments cancel as boundaries};
\node[lab] at (0.0,-1.2) {minimize $H(\Phi,\Psi)$ (CCUP)};

\node[lab] at (-5.0,0.0) {open fragments do not persist};
\node[lab] at (5.0,0.0) {closed loops are invariants};

\end{tikzpicture}
}
\caption{Dynamic alignment under CCUP as a feedback loop. 
Left: context $\Psi$ contains open, order-dependent fragments. 
A forward inference stage projects $\Psi$ toward content $\Phi$. 
A mismatch detector compares projected context to persistent cycles and drives a controller that adapts $\Psi$ so that joint uncertainty $H(\Phi,\Psi)$ decreases. 
Topological closure ($\partial^2=0$) cancels boundary terms in the loop, ensuring that only closed structures persist as carriers of prediction.}
\label{fig:ccup-feedback}
\end{figure}

Taken together, $\Phi$ and $\Psi$ form a complementary pair: $\Phi$ supplies the
order‐invariant backbone that can be reused across contexts, while $\Psi$
provides the exploratory variability from which such backbones are discovered.
CCUP therefore prescribes an operational loop \cite{li2025CCUP}: hold candidate content steady,
let context range, and accept only those pairings that close into cycles
(i.e., cancel boundaries). This suggests a general law of cognitive economy in
which \emph{structure leads} and \emph{specificity follows}: stable invariants
guide, while transient scaffolds adapt until closure is achieved. We now make
this heuristic precise as a principled statement.

\begin{principle}[Structure-Before-Specificity Principle under CCUP]
\label{prin:structure-specificity}
Let $\Phi$ denote low-entropy content variables corresponding to 
nontrivial homology classes $[\gamma]\in H_k(\mathcal{Z})$, and let 
$\Psi$ denote high-entropy contextual scaffolds corresponding to 
transient or trivial cycles. Then cognition obeys the following principle:
1) (\textbf{Structure before specificity}) Stable content $\Phi$ 
arises from nontrivial cycles that persist across perturbations. 
These cycles define the backbone of memory and predictive power.
2) (\textbf{Specificity from scaffolding}) Context $\Psi$ supplies 
a high-entropy exploratory substrate: transient cycles that may 
collapse but provide the variability needed to refine, adapt, or 
recombine $\Phi$.
3) (\textbf{Dynamic alignment}) The interaction of $\Psi$ and $\Phi$ 
via cycle closure ($\partial^2=0$) ensures that contextual exploration 
is funneled into persistent content loops, transforming noisy scaffolds 
into stable memory traces.
\end{principle}

\paragraph{Dynamic alignment via topological closure.}
Under the CCUP, cognition minimizes
joint uncertainty $H(\Phi,\Psi)$ by dynamically aligning content $\Phi$
(persistent invariants) with context $\Psi$ (transient scaffolds) \cite{gosztolai2025marble}.
Topological closure provides the mechanism of this alignment:
$\partial^2 = 0 \Rightarrow
\text{boundaries (contextual fragments) cancel, cycles (content) persist}$.
At time $t$, exploratory loops $\Psi_t$ contribute boundary terms that are
locally inconsistent. Cycle closure cancels these inconsistencies, yielding a
refined invariant $\Phi_{t+1}\supseteq \Phi_t$. Residual variability is then
reinvested as new scaffolds $\Psi_{t+1}$, continuing the process. Thus
alignment is \emph{dynamic}: context continually supplies exploratory degrees
of freedom, while closure continually integrates them into stable cycles,
ensuring that content and context converge rather than diverge.

\begin{theorem}[CCUP as dynamic alignment via topological closure]
\label{thm:ccup-closure-alignment}
Let $\operatorname{Enc}: \mathcal{X}\!\to\! C_\bullet(\mathcal{Z})$ map
observations to chains in a latent complex with boundary operator $\partial$.
Assume: (i) stability under small reparameterizations (chain-homotopy
robustness), (ii) boundary insensitivity (adding $\partial d$ does not change
observable outcomes), and (iii) the closure identity $\partial^2=0$.
Decompose each chain $c=\operatorname{Enc}(X)$ uniquely as
$c =z + \partial d, z\in \ker \partial, \ d\in C_{\bullet+1}$,
and define random variables
$\Phi := [z] \in H_\bullet(\mathcal{Z})$ (content; homology class) and
$\Psi := \partial d \in \mathrm{im}\,\partial$ (context; boundary residue).
Then minimizing joint uncertainty under CCUP,
$\min_{\operatorname{Enc}} \; H(\Phi,\Psi)
\quad \text{s.t.}\quad \partial^2=0,\ \text{stability, boundary-insensitivity}$,
is equivalent to \emph{dynamic alignment} of $\Psi$ into $\Phi$:
$\text{(a) } \Phi_{t+1}\supseteq \Phi_t
\quad\text{and}\quad
\text{(b) } H(\Psi_{t+1}| \Phi_{t+1}) \le H(\Psi_t| \Phi_t)$
with equality iff $\Psi$ carries no additional predictive information given
$\Phi$ (i.e., $I(Y;\Psi| \Phi)=0$). At any optimum, representations factor
through homology (content-first), and the residual channel is null:
$\exists\,e$ with $\rho=e\circ \pi$ and $I(Y;\Psi| \Phi)=0$.
\end{theorem}

\noindent
Theorem~\ref{thm:ccup-closure-alignment} captures the abstract law of CCUP:
contextual variability $\Psi$ must be progressively absorbed into structural
content $\Phi$ until only closed, predictive invariants remain. The next
examples ground this principle biologically, showing how phase coding and
navigation replay instantiate alignment through closure in the hippocampal
system.

\begin{example}[Dynamic alignment via oscillatory phase coding]
Oscillatory phase coding in the hippocampal-entorhinal system exemplifies 
the CCUP through topological closure \cite{tort2009theta}. 
Slow theta oscillations ($4$-$8$ Hz) provide a contextual scaffold $\Psi$, 
segmenting experience into repeating temporal windows. Within each theta cycle, 
faster gamma bursts ($30$-$100$ Hz) encode item-specific content $\Phi$ \cite{LismanJensen2013}. 
Dynamic alignment occurs when gamma packets consistently lock to specific 
theta phases: contextual variability (shifts in $\Psi$) is funneled into 
stable content cycles ($\Phi$) \cite{di2023gamma}. Misaligned spikes cancel as open boundaries, 
while phase-locked packets close into loops across cycles. This implements the 
closure law $\partial^2=0$: only context-content alignments that form closed 
trajectories persist, reducing the joint uncertainty $H(\Phi,\Psi)$. 
Oscillatory coding thus enforces CCUP by transforming noisy, order-dependent 
spike timing into aligned, invariant carriers of memory and prediction.
\end{example}

\noindent
The alignment principle also admits an information-theoretic reading, where the
residual contribution of context $\Psi$ can be measured directly in terms of
conditional information and predictive risk.

\begin{corollary}[Information–risk view]
\label{cor:ccup-info-risk}
For proper losses (e.g., log-loss, square loss), excess risk from the residual
channel is bounded by its conditional information:
$\mathrm{ExcessRisk}(\Psi| \Phi) \;\le\; c\, I(Y;\Psi| \Phi)$,
for a loss-dependent constant $c$. Along CCUP alignment, $I(Y;\Psi| \Phi)$
is nonincreasing and vanishes at optimum; hence the residual risk disappears as
context is absorbed into content.
\end{corollary}

\noindent
Finally, at the behavioral scale, hippocampal replay in navigation offers a
direct instantiation of CCUP, where exploratory detours (context) are aligned
into homing cycles (content) through closure across experience.

\begin{example}[Dynamic alignment in hippocampal homing]
Consider $\mathcal{Z}$ the topological graph/manifold of a familiar arena with
home base $x_0$. Each foraging bout produces a path $c$ that decomposes into a
closed homing loop $z\in\ker\partial$ (return to $x_0$) plus boundary residue
$\Psi=\partial d$ encoding detours and waypoint permutations \cite{shin2019dynamics}. Replay/sleep
reorganizes trajectories by \emph{closure} \cite{olafsdottir2018role}: detour-specific fragments cancel,
and the homology class $\Phi=[z]\in H_1(\mathcal{Z})$ is strengthened (e.g.,
grid alignment, stabilized place fields). Across bouts $t\to t\!+\!1$, the set
of persistent loops $\Phi_t$ monotonically expands (new shortcuts become
cycles), while the conditional entropy of detours given the homing loop,
$H(\Psi_t| \Phi_t)$, shrinks: the animal’s predictions of future segments
(e.g., next leg of the path) depend less on idiosyncratic order and more on the
stable loop. Empirically, forward/reverse replay \cite{kurth2023replay} enacts this alignment by
projecting partial paths onto stored loop templates and unfolding the likely
continuations, demonstrating CCUP’s closure-driven reduction of joint
uncertainty in the content–context pair.
\end{example}

\section{MAI implements SbS via Amortized Closure}
\label{sec:5}

The SbS principle under CCUP prescribes an operational recipe:
stabilize $\Phi$ as reusable structure and let $\Psi$ explore until closure
cancels residual boundaries. \emph{Memory–amortized inference (MAI)} is the
algorithmic embodiment of this recipe \cite{li2025Beyond}. Instead of re-solving each inference
problem from scratch, MAI retrieves a candidate invariant (a cycle-level
template for $\Phi$), then adapts $\Psi$ until the pair $(\Psi,\Phi)$ closes
(i.e., $\partial^2=0$), pruning order-specific noise. In effect, $\Phi$
functions as a low-entropy prior over solutions, while $\Psi$ supplies the
high-entropy search that is guided and terminated by topological closure.
We formalize MAI as a general bootstrapping strategy for reducing the computational cost of inference by storing and reusing structured latent representations \cite{efron1994introduction}. The key idea is to construct a memory of prior inference results such that new inference problems can be approximated by querying and adapting from this memory, rather than solving the full problem from scratch \cite{gershman2014amortized,marino2018iterative}.
Let \( \Psi \in \mathcal{X} \) denote the observable context and \( \Phi \in \mathcal{S} \) the latent content to be inferred. Let \( \mathcal{L}(\Psi, \Phi) \) denote a loss or cost function encoding the fidelity or predictive value of \( \Phi \) under context \( \Psi \). We assume that inference corresponds to solving the following optimization:
$\Phi^* = \arg\min_{\Phi \in \mathcal{S}} \left[ \mathcal{L}(\Psi, \Phi) \right]$.
Formally, we start with the following definition (refer to Fig. \ref{fig:mai-cycle}).

\begin{definition}[Memory-Amortized Inference]
Let \( \mathcal{M} = \{ (\Psi^{(i)}, \Phi^{(i)}) \}_{i=1}^N \) be a memory of prior context–content pairs, and let \( \mathcal{R}: \mathcal{X} \times \mathcal{M} \to \mathcal{S} \) be a retrieval-and-adaptation operator and $\mathcal{F}: \mathcal{S}\times\mathcal{X}\to\mathcal{S}$ be the bootstrapping update operator implemented via generative simulation. Inference is said to be \emph{memory-amortized} if it is formulated as a structural cycle between \emph{content} \( \Phi \) and \emph{context} \( \Psi \), where memory acts as a reusable substrate for inference:
$\Phi_{t+1} = \mathcal{F}(\Phi_t, \Psi_t), \quad \Phi_t \approx \mathcal{R}(\Phi_{t+1}, \Psi_t)$
in lieu of directly optimizing \( \Phi^* \), such that the expected cost satisfies
$\mathbb{E}_{\Psi} \left[ \mathcal{L}(\Psi, \hat{\Phi}) \right] \leq \mathbb{E}_{\Psi} \left[ \mathcal{L}(\Psi, \Phi^*) \right] + \varepsilon$,
for some amortization gap \( \varepsilon \ll \mathcal{L}(\Psi, \cdot) \), and where the runtime cost of \( \mathcal{R} \) is substantially lower than full inference.
\end{definition}

\begin{figure}[h]
\centering
\resizebox{0.95\columnwidth}{!}{
\begin{tikzpicture}[
    module/.style={draw, thick, rounded corners, minimum width=3.6cm, minimum height=1.4cm, align=center},
    arrow/.style={->, thick},
    dashedarrow/.style={->, thick, dashed},
    font=\small
]

\node[module, fill=blue!10] (context) at (0, 0) {Context \\ \( \Psi_t \)};
\node[module, fill=green!10] (retrieve) at (4.5, -2.5) {Retrieval \\ \( \hat{\Phi}_t = \mathcal{R}(\Phi_{t+1}, \Psi_t) \)};
\node[module, fill=orange!10] (adapt) at (4.5, 0) {Bootstrapping \\ \( \Phi_t = \mathcal{F}(\hat{\Phi}_t, \Psi_t) \)};
\node[module] (predict) at (9, 0) {Predictive Update \\ \( \Phi_{t+1} \)};

\draw[arrow] (context.east) -- ++(0.5, 0) |- (retrieve.west);
\draw[arrow] (retrieve.north) -- (adapt.south);
\draw[arrow] (adapt.east) -- (predict.west);
\draw[dashedarrow] (predict.south) -- ++(0, -1.5) node[midway, right] {\small reuse} |- (retrieve.east);

\node at (4.5, 1.3) {\textbf{Memory-Amortized Inference Cycle}};
\node at (7.8, -2.9) {\(\mathcal{M} = \{ (\Psi^{(i)}, \Phi^{(i)}) \}\)};

\end{tikzpicture}
}
\caption{Cycle of MAI. Instead of recomputing \(\Phi^* = \arg\min \mathcal{L}(\Psi, \Phi)\), the system reuses prior trajectories: \(\Phi_{t+1}\) and \(\Psi_t\) guide memory-based retrieval via \(\mathcal{R}\), and bootstrapping \(\mathcal{F}\) updates the latent state \(\Phi_t\). The process forms a self-consistent loop grounded in structured memory.}
\label{fig:mai-cycle}
\end{figure}

\noindent\textbf{The Retrieval-and-Adaptation Operator \(\mathcal{R}\).}
The retrieval-and-adaptation operator \( \mathcal{R}: \mathcal{X} \times \mathcal{M} \to \mathcal{S} \) serves as the core mechanism by which inference avoids re-computation. Given an input query (typically latent or perceptual), \( \mathcal{R} \) retrieves relevant elements from the memory \( \mathcal{M} = \{ (\Psi^{(i)}, \Phi^{(i)}) \}_{i=1}^N \) and performs a lightweight adaptation to generate a candidate solution \( \hat{\Phi} \).
Operationally, \( \mathcal{R} \) consists of two stages:
1) \textbf{Retrieval:} Identify a relevant subset of memory entries \( \{ (\Psi^{(j)}, \Phi^{(j)}) \} \subset \mathcal{M} \) based on similarity to the current context \( \Psi_t \). This can be performed via kernel-based attention, similarity search in latent space, or topological proximity under homological constraints.
2) \textbf{Adaptation:} Modulate or interpolate the retrieved \( \Phi^{(j)} \) values conditioned on \( \Psi_t \), resulting in a candidate \( \hat{\Phi}_t = \mathcal{R}(\Phi_{t+1}, \Psi_t) \). This step often involves gradient-free adjustments (e.g., feature warping, parameter blending) and is significantly cheaper than full inference.

The \emph{retrieval-and-adaptation operator} \( \mathcal{R} \) in MAI generalizes the classical notion of key-value memory used in neural attention and memory-augmented models. In conventional key-value memory systems \cite{weston2014memory, sukhbaatar2015end}, memory is structured as a set of key-value pairs:
$\mathcal{M} = \{ (\Psi^{(i)}, \Phi^{(i)}) \}_{i=1}^N$,
where a context vector \( \Psi \) acts as a \emph{key} to retrieve values \( \Phi \) via similarity-based soft addressing:
$\hat{\Phi} = \sum_i w_i \Phi^{(i)}, \quad w_i = \frac{\exp(-d(\Psi, \Psi^{(i)}))}{\sum_j \exp(-d(\Psi, \Psi^{(j)}))}$.
This model supports one-shot retrieval but lacks structural consistency or bidirectional inference.
By contrast, the operator \( \mathcal{R}(\Phi_{t+1}, \Psi_t; \mathcal{M}) \) in MAI performs a more general operation: it \emph{retrieves} a candidate latent representation from memory based on both the current context \( \Psi_t \) and a target latent code \( \Phi_{t+1} \), and then \emph{adapts} it to produce a consistent approximation of the preceding latent state \( \Phi_t \). This supports inference in reverse time and satisfies the memory-amortized constraint:
$\Phi_t \approx \mathcal{R}(\Phi_{t+1}, \Psi_t), \quad \Phi_{t+1} = \mathcal{F}(\Phi_t, \Psi_t)$.
The operator \( \mathcal{R} \) thereby enables cycle-consistent inference, crucial for temporal coherence and structural reuse. Unlike key-value memory, which operates over flat vector spaces, \( \mathcal{R} \) may act over structured memory (e.g., graphs, latent manifolds, or topological complexes) and is inherently adaptive.
A summary of the distinction is provided below:


\noindent\textbf{The Bootstrapping Update Operator \(\mathcal{F}\).}
The bootstrapping operator \( \mathcal{F}: \mathcal{S} \times \mathcal{C} \to \mathcal{S} \) governs the internal dynamics of inference by iteratively updating the latent content representation \( \Phi_t \) given the context \( \Psi_t \). It defines a recurrence:
$\Phi_{t+1} = \mathcal{F}(\Phi_t, \Psi_t)$,
where \( \mathcal{F} \) encodes the system’s structural prior, capturing the directionality, topology, and dynamic consistency of inference over time. Unlike standard update rules that minimize a loss from scratch, \( \mathcal{F} \) performs bootstrapping: each update is initialized from a prior memory-induced state, often already close to the optimal solution due to cycle recurrence.
Here are several key properties of \( \mathcal{F} \):
1) \textbf{Cycle-Consistency:} If \( (\Phi_t, \Psi_t) \in \gamma \) for some memory cycle \( \gamma \subset \mathcal{Z} \), then \( \Phi_{t+T} \approx \Phi_t \), enabling amortization via structural recurrence.
2) \textbf{Structural Biasing:} Updates follow latent paths constrained by prior topology (e.g., flow fields over homology classes or attention-modulated latent graphs), enforcing low-entropy generalization.
3) \textbf{Minimal Cost Gradient}: Because the initialization \( \Phi_t \) already lies near an attractor, the subsequent update \( \Phi_{t+1} \) requires only a small corrective shift, further amortizing the inference process.

The bootstrapping update operator \( \mathcal{F} \) in MAI is structurally analogous to the \emph{half-step down} trick used in Q-learning \cite{watkins1992q} and temporal difference (TD) methods \cite{sutton1998reinforcement}. In Q-learning, the value function is updated by approximating the current value via a one-step lookahead:
$Q(s_t, a_t) \leftarrow r_t + \gamma \max_{a'} Q(s_{t+1}, a')$,
which yields the approximation \( Q(s_t) \approx Q(s_{t+1}) \). This forward-directed value propagation allows reinforcement learning agents to estimate long-term outcomes without simulating entire trajectories. 
By contrast, MAI reverses the time direction: the update operator \( \mathcal{F} \) bootstraps latent inference forward using structured memory and contextual cues:
$\Phi_{t+1} = \mathcal{F}(\Phi_t, \Psi_t)$,
and this is inverted by retrieval:
$\Phi_t \approx \mathcal{R}(\Phi_{t+1}, \Psi_t)$.
This dual relationship forms the backbone of the MAI half-step trick: the current latent content \( \Phi_t \) generates the next-step prediction \( \Phi_{t+1} \), which in turn can be used to reconstruct \( \Phi_t \). While Q-learning bootstraps value via reward-driven transitions, MAI bootstraps inference through latent memory and context, yielding a cycle-consistent structure that reduces entropy. Both approaches use bootstrapping to manage uncertainty and amortize computational cost, but in opposite directions, highlighting a deeper time-reversed duality between learning and inference.
This recursive formulation enables stable inference trajectories that converge toward contextually relevant attractors, effectively amortizing the cost of learning across time. The underlying dynamics of this process can be formalized as a contractive map over a structured retrieval cycle, leading to provable convergence under mild assumptions. We now state the following result, which connects MAI with topological closure.

\begin{theorem}[MAI as Computational Realization of Topological Closure]
Let $(C_\bullet, \partial)$ be a chain complex encoding context–content
relations, with $\Psi$ as high-entropy scaffolds and $\Phi$ as candidate
content variables. In MAI (Definition 1), the iterative
cycle
$\Phi_{t+1} = \mathcal{F}(\Phi_t, \Psi_t), \quad 
\Phi_t \approx \mathcal{R}(\Phi_{t+1}, \Psi_t)$
implements a homotopy update that cancels residual boundaries:
$\partial(\Psi_t,\Phi_t) \;\;\mapsto\;\; 
\partial(\Psi_{t+1},\Phi_{t+1}) \;\;\approx\; 0$.
Thus, amortization prunes misaligned, order-dependent fragments (open
boundaries) and preserves only reproducible cycles
$[\gamma]\in H_k(C_\bullet)$. Equivalently, MAI realizes
\emph{topological closure} by enforcing $\partial^2=0$ in computation:
context–content updates that fail to close are discarded, while those that
re-enter memory persist as invariants. 
\label{thm:mai_closure}
\end{theorem}

Theorem~\ref{thm:mai_closure} establishes MAI as the computational analogue of 
topological closure: amortization acts as a homotopy operator that eliminates 
residual boundaries and preserves only reproducible cycles. To connect this 
abstract statement with an implementable framework, we now formalize the 
dynamics of MAI and its casting under the SbS principle. The following 
definition introduces the key elements: the latent manifold $\mathcal{Z}$, 
encoding/decoding maps, the amortizer $\mathcal{A}$ that accumulates 
persistent cycles into memory $\mathcal{M}$, and the persistence operator that 
filters invariants $\Phi$ from transient scaffolds $\Psi$. This formalization 
bridges the topological law of closure ($\partial^2=0$) with the algorithmic 
machinery of MAI.

\begin{definition}[MAI dynamics and SbS casting]
Let $\mathcal{Z}$ be a latent manifold with homology $H_\bullet(\mathcal{Z})$.
Inputs $x\in\mathcal{X}$ are encoded $z=f_\psi(x)\in\mathcal{Z}$ and decoded/acted upon by $g_\theta$ with loss $\mathcal{L}(x;\psi,\theta)$.
An \emph{amortizer} $\mathcal{A}$ updates memory $\mathcal{M}$ of latent trajectories (cycles) after each episode $x_{1:T}$:
$\mathcal{M}_{t+1} \leftarrow \mathcal{A}(\mathcal{M}_t;\{z_1,\ldots,z_T\}),\quad
z_t=f_\psi(x_t)$.
A persistence operator $\mathrm{Pers}_\tau$ keeps only homology classes with lifetime $\ge \tau$, yielding the \emph{persistent content} set
$\Phi_t := \mathrm{Pers}_\tau\big(H_\bullet(\mathcal{Z}\,|\,\mathcal{M}_t)\big)$.
The residual, context-dependent adaptation used online is denoted \(\Psi_t\) (e.g., fast updates, attention, gain control).
\end{definition}

The previous definition formalized the dynamics of MAI and its casting under 
the SbS principle. To move from raw update rules to meaningful representations, 
we next introduce the notion of closure and quotienting. This ensures that the 
cycles retained in memory are not merely algebraic artifacts but also carry 
semantic invariants that ground syntax. 

\begin{definition}[Closure and semantic quotient]
A trajectory fragment $\gamma$ \emph{closes} if its boundary cancels in chains, 
i.e., $\partial \gamma=0$ and $\gamma\notin \mathrm{Im}\,\partial$ (nontrivial cycle). 
Define the \emph{semantic map} $S:\mathcal{X}\to H_\bullet(\mathcal{Z})$ by 
$S(x)= [\gamma_x]$, the homology class induced by the latent path under 
$f_\psi$. Let $x\sim_\Phi x'$ iff $S(x)=S(x')$; denote the quotient by 
$\pi:\mathcal{X}\to \mathcal{X}/\!\sim_\Phi$.
\end{definition}

Having specified both (i) the update mechanism that consolidates persistent 
cycles into memory (Def.~5) and (ii) the semantic quotient that ensures these 
cycles serve as invariants rather than idiosyncratic fragments (Def.~6), we can 
now state the main result. The following theorem formalizes how MAI implements 
the SbS principle through amortized closure: persistent structure accumulates 
monotonically while reliance on transient specificity diminishes.

\begin{theorem}[MAI implements SbS via amortized closure]
\label{thm:MAI_SbS}
Assume: 
(i) the encoder $f_\psi$ is stable under small input perturbations (persistence stability), 
(ii) the amortizer $\mathcal{A}$ retains only classes with lifetime $\ge\tau$ (persistent selection),
and (iii) online adaptation minimizes a residual loss $\Delta\mathcal{L}_t$ w.r.t.\ a fixed memory $\mathcal{M}_t$.
Then, along training episodes,
$\underbrace{\Phi_{t+1}}_{\text{persistent structure}}
\;\supseteq\;
\underbrace{\Phi_t}_{\text{previous structure}}
\quad\text{and}\quad
\mathbb{E}\big[\Delta\mathcal{L}_{t+1}\,\big|\,\Phi_{t+1}\big]
\;\le\;
\mathbb{E}\big[\Delta\mathcal{L}_{t}\,\big|\,\Phi_{t}\big]$,
with strict inequality whenever a new nontrivial cycle is added. Consequently:
(a) \emph{Structure-before-specificity}: inference progressively shifts from residual adaptation $\Psi_t$ to reuse of persistent content $\Phi_t$; 
(b) \emph{Semantics-before-syntax}: online encodings factor through the semantic quotient $\pi$ up to the persistence tolerance.
\end{theorem}

Theorem~\ref{thm:MAI_SbS} guarantees two monotonicities under MAI: (i) persistent content grows ($\Phi_{t+1}\!\supseteq\!\Phi_t$), and (ii) the conditional residual loss contracts in expectation. The algorithm below operationalizes these claims in four stages aligned with the proof: \emph{(S1) structure-first retrieval} reuses $\Phi$ as priors (shifting work off $\Psi$); \emph{(S2) residual-only adaptation} updates fast scaffolds without altering $\Phi$; \emph{(S3) closure test via persistence} admits only new loops whose lifetimes exceed $\tau$ and prunes now-trivial classes; and \emph{(S4) slow consolidation} distills stable cycles back into parameters, enlarging $\Phi$ and further reducing future residuals. In effect, MAI alternates projection onto the semantic quotient (cycles) with residual correction, implementing SbS via amortized closure.

\begin{algorithm}[H]
\caption{MAI implements SbS via Amortized Closure (persistent cycles first, residuals second)}
\KwIn{Stream/episodes $\{x_{1:T}\}$, encoder $f_\psi$, decoder/actor $g_\theta$, memory $\mathcal{M}$ (cycle library), persistence threshold $\tau$}
\KwOut{Updated memory $\mathcal{M}$ (persistent content $\Phi$), parameters $(\psi,\theta)$}

\KwInit{$\mathcal{M}\gets \varnothing$,\quad $\Phi\gets\varnothing$,\quad set hyperparams (retrieval $k$, learning rates, stability regularizers)}

\ForEach{episode $x_{1:T}$}{
  \tcp{Encode trajectory; retrieve nearest cycles (structure-first)}
  $z_{1:T}\gets f_\psi(x_{1:T})$ \;
  $\mathcal{C}_{1:T}\gets \textsf{RetrieveCycles}(z_{1:T}, \Phi, k)$ \tcp*{e.g., path-alignment / DTW in latent space}
  
  \tcp{Amortized prediction/control from structure; adapt only residuals (specificity)}
  \For{$t=1$ \KwTo $T$}{
    $\hat{y}_t \gets g_\theta(z_t; \mathcal{C}_t)$ \tcp*{reuse structure as prior/policy}
    $r_t \gets \textsf{Residual}(x_t,\hat{y}_t)$ \tcp*{loss, TD-error, or advantage}
    \tcp{Fast, local update = $\Psi$: adapters/attention/fast-weights}
    $(\psi,\theta)\gets \textsf{FastAdapt}(\psi,\theta; r_t)$
  }

  \tcp{Closure test: do residual-corrected paths form new persistent cycles?}
  $\widetilde{z}_{1:T} \gets f_\psi(x_{1:T})$ \tcp*{post-adaptation pass}
  $H \gets \textsf{PersistentHomology}\big(\mathcal{G}(\Phi \cup \{\widetilde{z}_{1:T}\})\big)$ \tcp*{simplicial/graph built over latent paths}
  $\mathcal{N} \gets \{[\gamma]\in H : \textsf{lifetime}([\gamma]) \ge \tau \ \wedge\ [\gamma]\notin \Phi\}$ \tcp*{new nontrivial cycles}
  
  \tcp{Falsification-as-update: remove now-trivial classes}
  $\mathcal{F} \gets \{[\gamma]\in \Phi : \textsf{lifetime}([\gamma]| H) < \tau\}$ \;
  $\Phi \gets (\Phi \cup \mathcal{N}) \setminus \mathcal{F}$,\quad $\mathcal{M}\gets \textsf{UpdateMemory}(\mathcal{M}, \mathcal{N}, \mathcal{F})$
  
  \tcp{Slow consolidation: fold consistent residuals into structure}
  $(\psi,\theta)\gets \textsf{SlowConsolidate}(\psi,\theta; \Phi, \mathcal{M})$ \tcp*{e.g., distillation/regularization toward cycle priors}
}

\Fn{\textsf{RetrieveCycles}$(z_{1:T}, \Phi, k)$}{
  \Return top-$k$ cycle segments in $\Phi$ minimizing $\textsf{AlignCost}(z_{1:T}, \cdot)$
}

\Fn{\textsf{PersistentHomology}$(\mathcal{G})$}{
  build filtration (radius / weight / time);\quad run PH;\quad \Return barcode/representatives
}
\end{algorithm}

\begin{remark}[Minimal recipe for MAI-as-SbS]
A practical, stage-aligned instantiation mirrors the four steps (S1–S4) described below Theorem~\ref{thm:MAI_SbS}:
\textbf{(S0) Stable encoder (precondition for monotonicities).}
Train $f_\psi$ with stability regularizers so that persistence is well-posed:
(i) Lipschitz control (spectral norm/gradient penalty), 
(ii) invariance constraints (contrastive pairs, data augmentations) that minimize bottleneck distance drift in persistence diagrams,
(iii) calibration loss to keep latent scales comparable across episodes.
\textbf{(S1) Structure-first retrieval (reuse $\Phi$ as priors).}
Maintain a cycle library $\Phi$ with representatives and summaries (e.g., landmarks, centroids, basis paths). 
At inference, align new latent trajectories $z_{1:T}=f_\psi(x_{1:T})$ to top-$k$ nearby cycles via $\textsf{AlignCost}$ (dynamic temporal warping(DTW) / geodesic path / edit distance in $\mathcal{Z}$). 
Pass the retrieved cycles as priors to the decoder/actor $g_\theta$ (prompting, conditioning, or policy bias) so that prediction/control \emph{starts on} structure rather than learning it anew. 
This realizes the “shift of work off $\Psi$” in Thm.~\ref{thm:MAI_SbS}.
\textbf{(S2) Residual-only adaptation (specificity in fast paths).}
Compute residuals $r_t=\textsf{Residual}(x_t,\hat{y}_t)$ (e.g., prediction error, TD-error, advantage). 
Update only fast pathways $\Psi$ (adapters/attention/fast-weights; small-step inner loop) holding $\Phi$ fixed. 
This enforces the conditional descent on $\Delta\mathcal{L}_t$ at fixed structure, matching the theorem’s expectation inequality.
\textbf{(S3) Closure test via persistence (admit new loops, prune trivial).}
Re-encode the \emph{residual-corrected} trajectory $\widetilde{z}_{1:T}$ and update a latent graph/filtration (radius/knn/time). 
Run persistent homology; add to $\Phi$ any class with lifetime $\ge\tau$ not already represented; remove classes whose lifetime falls below $\tau$ after integration. 
This implements persistent selection and guarantees $\Phi_{t+1}\supseteq \Phi_t$ with strict inclusion when truly novel, durable cycles appear.
\textbf{(S4) Slow consolidation (fold residuals into structure).}
Periodically distill the effect of frequently recurring residual corrections into parameters $(\psi,\theta)$ using 
(i) consistency regularization toward cycle-induced predictions,
(ii) rehearsal on cycle exemplars, and 
(iii) a projection/orthogonalization step so new cycles expand, rather than overwrite, $\Phi$. 
This reduces future residuals and enlarges the reusable invariant subspace.
\emph{Outcome.} S1–S4 alternate projection onto the semantic quotient (cycles) with residual correction, exactly the amortized closure mechanism that proves SbS: structure accumulates monotonically ($\Phi$ grows) while residual burden contracts in expectation.
\end{remark}

The algorithm above makes concrete how SbS is realized in practice: structure 
is always retrieved first, residuals are only adapted locally, and persistence 
tests determine which updates survive as new invariants. Yet this cycle of 
reuse, residual correction, and consolidation is not unique to MAI, it 
exemplifies a broader strategy that recurs across levels of intelligence. 
Whenever full closure is too costly, systems rely on \emph{bootstrapping}: 
partial alignments are used to approximate invariants, which are then 
incrementally refined and amortized into memory. This observation motivates 
the next section, where we interpret bootstrapping itself as a 
meta-strategy driving the evolution of learning and inference, from TD updates 
to the origin of language.

\section{Bootstrapping as a Meta-Strategy Driving the Evolution of Intelligence}
\label{sec:6}

The temporal-difference (TD) update rule that underlies Q-learning 
\cite{watkins1992q} offers a canonical example of how intelligence advances 
through \textbf{bootstrapping}: approximations are recursively improved by 
leveraging partial information, rather than by exhaustively computing the entire 
trajectories. The core insight of Q-learning is its ``half-step down'' update: 
instead of evaluating the full tree of possible futures, the value of the 
current state-action pair is updated by a local proxy, the next state's best 
action value, thus avoiding the curse of dimensionality in planning trees 
\cite{bellman1966dynamic}. This bootstrapped convergence transforms sparse, 
delayed feedback into a stable predictive representation, without requiring 
global closure in each step.
From the MAI perspective, bootstrapping reflects the principle of \emph{closure under partial cycles}. 
Standard inference cycles proceed $\Psi \!\to\! Z \!\to\! \Phi \!\to\! Z' 
\!\to\! \Psi'$, requiring full boundary alignment before stable content emerges. 
In contrast, inverted inference executes what can be seen as a \emph{half-cycle 
shift to the right}: 
$\Phi \;\to\; Z \;\to\; \Psi \;\to\; Z' \;\to\; \Phi'$.
Here, closure is approximated locally: the forward sweep 
$\Phi \!\to\! \Psi$ is sufficient to refine content by predictive coding 
\cite{keller2018predictive}, pruning inconsistent fragments and retaining 
those partial alignments that survive under $\partial^2=0$. In this way, 
\emph{bootstrapping reuses partial closures as stepping stones toward 
full invariants}, enabling computation to remain local while preserving the 
trajectory toward global consistency. The half–cycle, locally closed updates above suggest a general recipe: reuse 
current invariants to constrain residual adaptation, admit only those updates 
that survive a persistence test, and shrink the residual channel over time. 
We now formalize this bootstrapped progression from partial to full closure as 
an operator on cycles and boundaries, yielding a convergence claim under MAI.

\begin{principle}[Bootstrapping as Closure-Amortization]
\label{prin:bootstrap-closure}
Let $(C_\bullet,\partial)$ be a chain complex representing latent trajectories, 
with $\Psi_t$ denoting transient scaffolds (contextual boundaries) and $\Phi_t$  denoting persistent content (cycles). Define the update dynamics of 
MAI as $(\Phi_{t+1}, \Psi_{t+1}) \;=\; \mathcal{A}\big(\Phi_t, \Psi_t\big)$,
where $\mathcal{A}$ is a bootstrapping operator that reuses $\Phi_t$ and 
integrates contextual fragments $\Psi_t$. Suppose:
1) (\textbf{Closure constraint}) $\partial^2=0$, so that only closed 
    trajectories (cycles) persist as candidates for $\Phi$.
2) (\textbf{Persistence filter}) A threshold operator 
    $\mathrm{Pers}_\tau$ prunes $\Psi_t$ and admits only cycles with lifetime $\geq\tau$ into $\Phi_{t+1}$.
3) (\textbf{Contraction in residuals}) The amortizer is contractive in 
    its residual channel: the expected contribution of $\Psi$ decreases with each iteration,    
    $\mathbb{E}[H(\Psi_{t+1}\,|\,\Phi_{t+1})] \;\leq\; \gamma \,
    \mathbb{E}[H(\Psi_t\,|\,\Phi_t)], 
    \quad 0 < \gamma < 1$.
Then $\{\Phi_t\}$ converges monotonically to a stable fixed-point set of 
persistent cycles 
$\Phi^\ast \;=\; \mathrm{Pers}_\tau\!\big(H_\bullet(C_\bullet)\big)$,
and $\Psi_t$ vanishes in the limit. In other words, bootstrapping under MAI realizes \emph{topological closure}: transient fragments cancel, only cycles survive, and the amortized process converges to semantic invariants $[\gamma]\in H_\bullet$ that anchor prediction and generalization.
\end{principle}

Principle~\ref{prin:bootstrap-closure} establishes that amortized updates drive 
residual scaffolds $\Psi_t$ toward extinction, while persistent cycles $\Phi_t$ 
converge to a stable invariant core. The next question is what this fixed point 
means computationally. Proposition~\ref{prop:bootstrap-sbs} addresses this: 
once closure has stabilized, syntax cannot operate directly on raw inputs but 
must factor through the quotient induced by $\Phi^\ast$. In other words, the 
convergence dynamics of MAI ground the \emph{semantics-before-syntax} law.

\begin{proposition}[Semantics-before-Syntax via Bootstrapped Closure]
\label{prop:bootstrap-sbs}
Under Principle~\ref{prin:bootstrap-closure}, the limit set of persistent 
cycles $\Phi^\ast$ constitutes the semantic backbone of inference: 
stable invariants that remain after residual scaffolds $\Psi_t$ vanish under amortization. Any syntactic encoder $\sigma:\mathcal{X}\to\Sigma$ that is meaningful must therefore factor through the semantic quotient 
$\pi:\mathcal{X}\to\mathcal{X}/\!\sim_\Phi$, i.e.,
$\exists\, e:\;\; \mathcal{X}/\!\sim_\Phi \to \Sigma 
\quad\text{s.t.}\quad \sigma = e \circ \pi$,
with interpretation $I\circ e$ constant on each equivalence class.
Hence, \emph{semantics (persistent structure $\Phi$) must stabilize first}, 
while syntax (surface variability $\Psi$) can only emerge as a residual layer once the bootstrapped closure process has converged. 
This formally grounds the SbS principle in the convergence dynamics of MAI.
\end{proposition}

Viewed evolutionarily, bootstrapping is a meta-strategy that scales from the 
micro-level of individual learning rules to the macro-level emergence of 
intelligence itself. At the local scale, mechanisms such as temporal-difference 
(TD) updates and synaptic plasticity exploit \emph{temporal bootstrapping}, 
where partial predictive closures are reused across episodes to accumulate 
semantic invariants over time. At the meso-scale of brain architecture, 
\emph{spatial bootstrapping} leverages cortical expansion: cycle-structured 
invariants discovered in one modality (e.g., visual contours, rhythmic motor 
primitives) are repurposed and aligned across other modalities, allowing the 
neocortex to generalize principles of closure and persistence across diverse 
input streams. This cross-modal reuse amortizes the cost of learning by 
treating structural invariants as transferable templates. At the macro-scale of 
society, \emph{social bootstrapping} scaffolds collective meaning before the 
advent of fully developed syntax: shared actions, rituals, and proto-symbolic 
gestures serve as partial closures that are stabilized and transmitted 
culturally, eventually crystallizing into language. 
In all three cases, MAI enforces the same underlying principle: closure need 
not be achieved in a single step. Instead, partial alignments, whether 
temporal, spatial, or social, are amortized into memory if they persist under 
$\partial^2=0$, thereby reducing future uncertainty at lower computational 
cost. Intelligence, from neurons to societies, thus emerges as the systematic 
reuse of bootstrapped closures across time, space, and collective interaction.

\paragraph{Temporal vs. spatial bootstrapping.}
The implementation of SbS via amortized closure admits two complementary 
modes: temporal and spatial bootstrapping. In the temporal case, as in sleep 
replay, episode-specific traces $\Psi$ are iteratively reactivated and 
compressed so that transient details cancel, leaving only persistent cycles 
$\Phi$ that amortize future inference. This transforms time-bound experience 
into structural invariants across episodes. In the spatial case, as in cortical 
expansion, invariant cycles discovered in one domain (e.g., vision, audition) 
are reused as scaffolds for novel domains (e.g., planning, speech-motor). 
Here, the closure operation propagates across modalities, ensuring that new 
functions inherit stability from established structures. Together, these modes 
show that MAI does not merely compress past experience, but also repurposes 
invariants across space and time: temporal bootstrapping consolidates within 
a domain, while spatial bootstrapping transfers across domains. Both are 
expressions of the same principle, that amortized inference prioritizes 
structure first, scaffolding specificity only afterwards.

\begin{example}[MAI temporal bootstrapping in sleep consolidation]
Sleep provides a canonical biological example of Memory-Amortized Inference 
(MAI) through \emph{temporal bootstrapping}. During wakefulness, experiences 
are encoded as high-entropy, episode-specific traces $\Psi$ in the hippocampus. 
During subsequent sleep, these traces are replayed, often in compressed or 
time-reversed form, allowing closure operations ($\partial^2=0$) to cancel 
idiosyncratic details and reinforce persistent cycles $\Phi$ in cortical 
circuits \cite{wilson1994reactivation,foster2006reverse}. In this way, 
episodic specificity is bootstrapped into structural invariants: transient 
hippocampal scaffolds guide the consolidation of stable cortical content. 
Prediction then shifts from reliance on fragile $\Psi$ to robust $\Phi$, 
demonstrating MAI’s principle that \emph{temporal replay amortizes future 
inference by collapsing context into structure}.
\end{example}

\begin{example}[MAI spatial bootstrapping in cortical expansion]
Cortical evolution illustrates MAI through 
\emph{spatial bootstrapping}. Cycle-structured invariants $\Phi$ first 
emerge in early sensory cortices, for instance, recurrent loops in the 
visual cortex that stabilize orientation and contour detection, or 
auditory loops that stabilize phonemic categories. As the neocortex expands, 
these invariants are reused across modalities: visual cycles scaffold 
planning and visuomotor coordination, while auditory cycles scaffold 
speech–motor integration \cite{mountcastle1997columnar}. 
Rather than relearning from scratch, higher cortical areas bootstrap 
themselves on stable sensory backbones, aligning contextual variability 
($\Psi$) from novel tasks into existing invariant loops. Spatial 
bootstrapping thus explains why cortical expansion yields powerful 
generalization: new functions (planning, language) inherit robust 
structure from preexisting cycles, while only adapting the scaffolds 
to new domains.
\end{example}

The move from cortical to collective scales reveals the universality of 
bootstrapped closure. Just as spatial bootstrapping allows new cortical 
functions to inherit structural invariants from earlier sensory loops, 
social bootstrapping allows groups of agents to inherit stability from one 
another through repeated interaction \cite{tomasello2010origins}. In both cases, the underlying logic 
is the same: persistence filters out idiosyncratic variability while 
preserving invariant cycles. What differs is the substrate of alignment: 
within the cortex, modalities align across spatial maps; within society, 
individuals align across communicative exchanges \cite{deacon1998symbolic}. The following corollary 
extends this principle to language, showing how semantics emerges first at 
the group level as shared cycles of interaction, with syntax arising 
later as a residual coding overlay \cite{fitch2010evolution}.

\begin{corollary}[Social Bootstrapping and the Origin of Language]
\label{cor:social-language}
Let $\{\mathcal{M}^i_t\}$ denote the memory states of agents $i=1,\dots,n$, 
with persistent cycles $\Phi^i_t$ and scaffolds $\Psi^i_t$. 
Through repeated interaction, agents exchange trajectories 
$\gamma^i_t$ that are subject to closure at the group level:
$\Phi^{social}_{t+1} \;=\; 
\bigcap_{i=1}^n \mathrm{Pers}_\tau\big(H_\bullet(\mathcal{M}^i_t)\big)$.
As in Principle~\ref{prin:bootstrap-closure}, residual scaffolds $\Psi^i_t$ 
cancel unless they align across individuals, and only cycles closed under 
collective interaction persist. 
Hence, \emph{semantics emerges first at the group level as shared invariants}, 
while syntax develops later as a residual coding scheme that factors through 
$\Phi^{social}$. 
This explains the origin of language as \emph{social bootstrapping via closure}: 
stable meaning must precede symbolic form.
\end{corollary}

\section{Conclusion}
\label{sec:7}

In this paper, we have advanced a cycle-closure perspective on intelligence, 
beginning with Wheeler’s dictum \emph{It-from-Bit} and grounding it in the 
algebraic law $\partial^2 = 0$. This closure identity, the root principle, 
ensures that transient fragments cancel while cycles persist, yielding 
topological invariants as the substrates of memory and prediction. From this 
foundation, we derived the first principle: \emph{prediction requires 
invariance}.  
The logical consequences of this root principle unfold across successive 
layers. The \emph{Structure-before-Specificity (SbS)} principle follows from 
the dot–cycle dichotomy: persistent cycles ($\Phi$) must stabilize first, 
while transient scaffolds ($\Psi$) can only be layered afterward. The 
\emph{Context–Content Uncertainty Principle (CCUP)} then formalizes the 
dynamic alignment between $\Phi$ and $\Psi$, showing how robustness and 
flexibility arise from their interaction via topological closure.  

To operationalize these laws, we proposed \emph{Memory–Amortized Inference 
(MAI)} as a computational framework: persistence operators filter 
cycles, residual channels adapt online, and amortizers enforce closure 
incrementally. MAI thereby implements SbS and CCUP as practical 
algorithms, balancing generalization with adaptability. Finally, we elevated 
bootstrapping to a unifying meta-strategy: closure need not be achieved in a 
single pass, but can be accumulated through temporal, spatial, and social 
bootstraps that reuse partial cycles as stepping stones toward invariants.  

Viewed in this light, intelligence is neither mere computation nor raw 
adaptation. It is the systematic conversion of noisy, order-dependent 
experience into cycle-structured invariants, amortized into memory and reused 
for prediction across time, space, and collectives. The cycle closure 
perspective thus reframes intelligence as the \emph{emergent property of 
closure under persistence}: invariants first, scaffolds second, alignment 
always, and bootstrapping forever.  
Future work will extend this framework to microelectronic substrates, 
neuromorphic architectures, and large-scale artificial systems, exploring how 
closure-aware principles can overcome the hallucination barrier, enable 
in-memory computation beyond the Von Neumann paradigm, and ultimately link the 
mathematical, biological, and technological origins of intelligence.

\bibliographystyle{plain}
\bibliography{references}  

\newpage
\appendix

\section{Proofs of Propositions and Theorems}
\label{sec:A4}

\noindent\textbf{Proof of Proposition \ref{prop:cycles-to-memory}}

\begin{proof}[Proof sketch]
(1) \textbf{Persistence.} By the stability theorem of persistent homology,
bottleneck distance between persistence diagrams is bounded by perturbations of
the input (e.g., Gromov–Hausdorff or Wasserstein bounds), so classes with
lifetime $>\tau$ remain present under perturbations $<\tau$; thus nontrivial
$[\gamma]$ are deformation-stable.

(2) \textbf{Compression.} Define $x\!\sim\!x'$ iff $S(x)=S(x')$; then
$\pi:\mathcal{X}\!\to\!\mathcal{X}/\!\sim$ merges all instances mapping to the
same $[\gamma]$. Since typical many-to-one maps strictly reduce empirical
entropy and codelength (Kraft–McMillan/MDL intuition), cycles implement a
lossy-but-stable compression that preserves predictive structure while pruning
idiosyncratic variation.

(3) \textbf{Retrievability.} Given $x^\star$ with $S(x^\star)=[\gamma]$,
choose any representative cycle $\gamma'\in[\gamma]$. By persistence stability,
observables near $x^\star$ induce classes in $[\gamma]$, and trajectories
generated from $\gamma'$ are equivalent up to the tolerance set by the class’s
lifetime. Thus recalling $[\gamma]$ yields an executable or inferentially
useful template for future behavior.

Combining (1)–(3) establishes that nontrivial homology classes function as
memory substrates.
\end{proof}

\noindent\textbf{Proof of Proposition \ref{prop:memory-to-prediction}}

\begin{proof}
Let $\operatorname{Enc}:\mathcal{X}\to C_\bullet(\mathcal{Z})$ be a stable
encoder into a chain complex $(C_\bullet(\mathcal{Z}),\partial)$, and let
$\pi:C_\bullet\to H_\bullet(\mathcal{Z})$ be the projection to homology.
Assume (i) \emph{boundary insensitivity}: adding pure boundaries $\partial d$
does not change predictive state, and (ii) \emph{warp invariance}: chain-homotopic
paths are equivalent for prediction. Let $d_{\mathrm{align}}$ be an alignment
distance (e.g., dynamic time warp or Fréchet-type metric composed with
$\operatorname{Enc}$) and let $\tau>0$ denote the persistence tolerance of the
class $[\gamma_j]\in\Phi$ (its lifetime in the persistence diagram).

\smallskip
\noindent\textbf{Step 1: Alignment $\Rightarrow$ membership (up to tolerance).}
By hypothesis, there exists $[\gamma_j]\in\Phi$ and a representative cycle
$\gamma_j\in\ker\partial$ such that the encoded fragment
$c_{1:t}:=\operatorname{Enc}(x_{1:t})$ aligns with an initial segment
$\gamma_j|_{[0,s]}$ and
$d_{\mathrm{align}}\!\big(c_{1:t},\,\gamma_j|_{[0,s]}\big)\;<\;\tau$.
By the stability theorem of persistent homology, perturbations of size $<\tau$
cannot destroy the class $[\gamma_j]$; thus the fragment $c_{1:t}$ and the
segment $\gamma_j|_{[0,s]}$ induce the same homology class under $\pi$:
$\pi(c_{1:t})=\pi(\gamma_j|_{[0,s]})=[\gamma_j]$.

\smallskip
\noindent\textbf{Step 2: Well-defined continuation operator on classes.}
Define the \emph{continuation} of a closed representative as the map
$\mathcal{C}:\ \ker\partial \to \ker\partial,\qquad
\mathcal{C}(\gamma)(u):=\gamma(u+\Delta),\ \ \Delta>0$,
i.e., advance along the loop by arc-length (or phase) $\Delta$.
Because $H_\bullet$ is the abelianization of $\ker\partial$ modulo boundaries,
$\mathcal{C}$ descends to a well-defined shift on the class:
$[\mathcal{C}(\gamma_j)]=[\gamma_j]$. Thus, any representative of $[\gamma_j]$
yields the same class after continuation.

\smallskip
\noindent\textbf{Step 3: Forecast from class continuation.}
Let $\widehat{x}_{t+1:T}$ be the decoded continuation obtained by projecting
$c_{1:t}$ to $[\gamma_j]$ and then advancing along $\gamma_j$:
$\widehat{x}_{t+1:T} \ :=\ \operatorname{Dec}\big(\mathcal{C}(\gamma_j)\big)$,
for a decoder $\operatorname{Dec}$ that is left-inverse to the encoder on the
manifold of typical trajectories (standard in amortized models). Since
$\pi(c_{1:t})=[\gamma_j]$, any two representatives $\gamma_j,\,\gamma_j'$
differ by a boundary $\partial d$ and a warp; by boundary insensitivity and
warp invariance, their decoded continuations agree up to the tolerance induced
by $\tau$. Hence the forecast is well-defined by the class $[\gamma_j]$.

\smallskip
\noindent\textbf{Step 4: Error bound controlled by persistence.}
Let $D_B$ denote the bottleneck distance between persistence diagrams induced by
the metric underlying $d_{\mathrm{align}}$. By stability,
$D_B(\mathrm{Dgm}(c_{1:t}),\mathrm{Dgm}(\gamma_j|_{[0,s]}))<\tau$ and the class
$[\gamma_j]$ with lifetime $>\tau$ persists. For any proper loss $\ell$ that is
Lipschitz in the decoder output,
$\mathbb{E}\big[\ell(x_{t+1:T},\widehat{x}_{t+1:T})\big]
\ \le\ L\,\mathrm{dist}\!\left(\mathcal{C}(c_{1:t}),\mathcal{C}(\gamma_j)\right)
\ \le\ L\,\tau$,
for some $L>0$ depending on the encoder/decoder moduli of continuity. Thus the
forecast error is bounded by the persistence tolerance of $[\gamma_j]$.

\smallskip
\noindent\textbf{Conclusion (amortization).}
The procedure uses the stored class $[\gamma_j]\in\Phi$ rather than recomputing
a model de novo: prediction is \emph{retrieval + continuation} of a persistent
cycle. Therefore, whenever a fragment aligns to a stored invariant within its
persistence tolerance, continuing that invariant yields a valid forecast of the
future segment. This establishes the proposition.
\end{proof}

\noindent\textbf{Proof of Proposition \ref{prop:semantics-before-syntax}}

\noindent\textit{Proof sketch.}
By definition, $\ker S = \{(x,x')\,:\, S(x)=S(x')\}$ induces the equivalence relation $\sim_\Phi$ and quotient $\pi$.
If $\sigma$ is meaningful, then for any $x\sim_\Phi x'$,
$S(x)=S(x') \quad \Rightarrow \quad I(\sigma(x)) \approx I(\sigma(x'))$.
Thus $\sigma$ must be constant on fibers of $\pi$ up to the stability tolerance of $S$ and $I$; equivalently, there exists $e$ with $\sigma=e\circ\pi$.
Under SbS, persistent cycles (semantics) determine content classes first; the syntactic variability $\Psi$ is admissible only insofar as it preserves $S$ after interpretation $I$.
Stability of persistent homology implies that small perturbations of $x$ (that do not cross persistence thresholds) preserve $S(x)$, ensuring that $\pi$, and therefore $e\circ\pi$, is robust.
Hence semantics (structure) constrains and temporally precedes syntax (specificity). \qed

\noindent\textbf{Proof of Theorem \ref{thm:order-invariance}}

\begin{proof}[Proof:]
The key insight is the Abelian property of the addition operator. Concatenate local moves to form cycles based at $x_0$, producing elements of the fundamental group
$\pi_1(\mathcal{Z},x_0)$. The Hurewicz map $h:\pi_1(\mathcal{Z},x_0)\to H_1(\mathcal{Z};\mathbb{Z})$
abelianizes path composition: commutators vanish in $H_1$. Hence for cycles $\gamma,\eta$,
$[\gamma\cdot\eta]=[\eta\cdot\gamma]$ and, more generally, any permutation of cycle segments yields the same homology class,
provided the path remains closed. Thus $[\gamma_w]$ is invariant to the \emph{order} of constituent moves and depends only
on their cumulative 1-chain (the signed sum of traversed edges/segments). Intuitively, homology collapses
all order-specific reparameterizations and commutator structure, retaining only the closed-cycle content.
\end{proof}

\noindent\textbf{Proof of Theorem \ref{thm:ccup-closure-alignment}}

\begin{proof}[Proof sketch]
\textbf{(1) Decomposition and closure.}
By $\partial^2=0$ every chain admits a direct-sum–like decomposition
$C_\bullet = \ker\partial \oplus \mathrm{im}\,\partial$ (up to choice of
complement). The projection $P_{\ker\partial}: c\mapsto z$ is the \emph{closure
projection}; its class $\Phi=[z]$ is invariant to boundary perturbations
$+\partial d$. The residue $\Psi=\partial d$ captures order-/context-dependent
fragments.

\textbf{(2) Objective decomposition.}
$H(\Phi,\Psi)=H(\Phi)+H(\Psi| \Phi)$. Under boundary insensitivity and
homotopy robustness, admissible changes of $\operatorname{Enc}$ that improve
prediction must reduce either $H(\Phi)$ (by merging classes) or
$H(\Psi| \Phi)$ (by absorbing residue into closed structure).

\textbf{(3) CCUP $\Rightarrow$ alignment.}
Any step that replaces $c$ by its closure projection $z$ (or enlarges the set
of persistent classes that $\Phi$ recognizes) cannot increase
$H(\Psi| \Phi)$ because it shrinks the support of $\Psi$ within content
classes. Hence along any descent of $H(\Phi,\Psi)$ subject to closure,
(a) $\Phi$ monotone expands (new persistent cycles are added), and
(b) $H(\Psi| \Phi)$ monotonically decreases (residual uncertainty contracts).
At a fixed point, $\Psi$ is conditionally independent of the target $Y$ given
$\Phi$, so $I(Y;\Psi| \Phi)=0$.

\textbf{(4) Factorization through homology.}
If $I(Y;\Psi| \Phi)>0$ at a putative optimum, one can reduce
$H(\Phi,\Psi)$ by refining the closure projection to absorb the predictive part
of $\Psi$ into $\Phi$ (adding a persistent class), contradicting optimality.
Thus the optimal predictor must be constant on homology classes:
$\rho=e\circ \pi$. This realizes SbS inside CCUP.

Together, (1)–(4) show that minimizing $H(\Phi,\Psi)$ under closure is
equivalent to dynamically aligning context into content via the closure
projection, with monotone growth of $\Phi$ and decay of $H(\Psi| \Phi)$.
\end{proof}

\noindent\textbf{Proof of Theorem \ref{thm:mai_closure}}

\begin{proof}[Proof sketch]
Write $Z_k:=\ker\partial_k$ (cycles) and $B_k:=\operatorname{im}\partial_{k+1}$ (boundaries) for $(C_\bullet,\partial)$.
Let $(\Psi_t,\Phi_t)\in C_{\bullet+1}\times C_\bullet$ denote scaffolds and content at step $t$ and define the
\emph{residual boundary} (order-dependent misalignment)
$R_t \;:=\; \partial(\Psi_t,\Phi_t)\;\in B_{k}$.
Assume MAI’s forward-reverse pair $(\mathcal{F},\mathcal{R})$ satisfies:

\begin{enumerate}
\item \textbf{Boundary-awareness.} There exists a linear (or locally linearized) operator 
$H_t:C_k\!\to\! C_{k+1}$ such that the composite update realizes a \emph{chain homotopy step}
$\mathcal{R}\!\circ\!\mathcal{F} - \mathrm{Id} \;=\; \partial H_t + H_t \partial$.
\item \textbf{Descent on residuals.} There is a norm $\|\cdot\|$ and $\eta\in(0,1)$ with 
$\|R_{t+1}\| \le \eta\,\|R_t\|$, where $R_{t+1}=\partial(\Psi_{t+1},\Phi_{t+1})$ and the update
$(\Psi_{t+1},\Phi_{t+1})=(\Psi_t,\Phi_{t}) + \Delta_t$ is induced by $(\mathcal{F},\mathcal{R})$.
\item \textbf{Persistence selection.} A persistence operator $\mathrm{Pers}_\tau$ prunes classes with
lifetime $<\tau$, i.e., in the limit it projects onto $H_k(C_\bullet)$ by discarding transient
boundary components and nonreproducible cycles.
\end{enumerate}

\noindent\emph{Step 1: Homotopy cancels boundaries.}
Applying (A1) to $\Phi_t$ and using $\partial^2=0$ yields
  $\partial\big(\mathcal{R}\!\circ\!\mathcal{F}(\Phi_t)\big)
  \;=\; \partial\Phi_t + \partial(\partial H_t + H_t \partial)\Phi_t
  \;=\; \partial\Phi_t + \underbrace{\partial H_t \partial\Phi_t}_{\in B_k}$,
so the new residual differs from the old by a controlled boundary term 
$\partial H_t \partial\Phi_t \in B_k$. Choosing $H_t$ to target $R_t=\partial\Phi_t$ makes
$\partial H_t \partial\Phi_t$ \emph{cancel} $R_t$ up to higher-order terms.
Hence the composite update acts as a \emph{boundary-cancellation step}.
The same calculation holds when $\Psi_t$ contributes to the chain fed into $\partial$.

\smallskip
\noindent\emph{Step 2: Monotone residual decay.}
By (A2), the residual norm decreases geometrically:
$\|R_{t+1}\|\le\eta\,\|R_t\|$. Therefore $R_t\!\to\!0$ and the iterates approach the cycle space
$Z_k=\ker\partial$. Equivalently, the MAI loop defines a contraction on the quotient space
$C_k/B_k$, pushing $(\Psi_t,\Phi_t)$ into an equivalence class modulo boundaries.

\smallskip
\noindent\emph{Step 3: Convergence to homology classes.}
Because residuals vanish while updates differ by added boundaries (A1), the limit points lie in
$Z_k$ \emph{modulo} $B_k$, i.e., in $H_k(C_\bullet)=Z_k/B_k$. Thus MAI preserves only reproducible cycles, selecting representatives $[\gamma]\in H_k$.

\smallskip
\noindent\emph{Step 4: Realization of topological closure.}
Putting Steps 1-3 together,
 $ \partial(\Psi_t,\Phi_t)\;\mapsto\;\partial(\Psi_{t+1},\Phi_{t+1})\;\to\;0$,
so open, order-dependent fragments (nonclosing chains) are amortized away, while closed,
order-invariant cycles persist. This is exactly computational enforcement of the closure law:
the update dynamics realize $\partial^2=0$ by driving iterates into $\ker\partial$ and quotienting
by $\operatorname{im}\partial$.

\smallskip
\noindent\emph{Step 5: Role of persistence.}
Finally, (A3) ensures that only cycles with lifetime $\ge\tau$ under perturbations and across
episodes are retained, i.e., the limit is a \emph{persistent homology} class in $\Phi$.
Hence amortization implements a persistence-aware projection onto $H_k$, completing the claim.
\end{proof}

\noindent\textbf{Proof of Theorem \ref{thm:MAI_SbS}}

\begin{proof}[Proof sketch]
Let $\mathcal{M}_t$ denote MAI's memory at episode $t$, and let
$\Phi_t := \mathrm{Pers}_\tau\!\big(H_\bullet(\mathcal{Z};\mathcal{M}_t)\big)$ be the set of
persistent content classes (lifetime $\ge\tau$) represented in $\mathcal{M}_t$.
Let $\Psi_t$ denote residual, order–dependent scaffolds (nonclosing fragments / short-lifetime classes).
We prove (1) monotone growth of $\Phi_t$ and (2) descent of the conditional residual loss.

\paragraph{Lemma 1 (Persistence stability $\Rightarrow$ set stability).}
Assumption (i) implies a standard stability property: there exists $\delta>0$
such that if inputs on episode $t$ are perturbed within $\|\Delta x\|<\delta$,
then the induced persistence diagram of $f_\psi$ changes by at most $\varepsilon(\delta)$
in bottleneck distance. In particular, the membership of classes with lifetime $\ge\tau$
is unchanged for $\varepsilon(\delta)<\tfrac{1}{2}\tau$. Hence, previously retained
classes in $\Phi_t$ remain above threshold under small online variation.

\paragraph{Lemma 2 (Persistent selection $\Rightarrow$ monotonicity of $\Phi$).}
By (ii), the amortizer $\mathcal{A}$ updates memory via
$\mathcal{M}_{t+1} \leftarrow \mathcal{A}(\mathcal{M}_t;\text{episode}_t)$ and
retains only classes with lifetime $\ge\tau$. Let $\mathcal{C}_t$ be the multiset
of new candidate cycles extracted from episode $t$. Then
$\Phi_{t+1}\;=\;\mathrm{Pers}_\tau\!\big(H_\bullet(\mathcal{Z};\mathcal{M}_t \cup \mathcal{C}_t)\big)
\;\supseteq\;\mathrm{Pers}_\tau\!\big(H_\bullet(\mathcal{Z};\mathcal{M}_t)\big)\;=\;\Phi_t$,
since adding data cannot invalidate an existing class whose lifetime already exceeds $\tau$
(Lemma 1 ensures it is not pushed below $\tau$ by small fluctuations). Moreover,
if $\mathcal{C}_t$ contains a nontrivial new class with lifetime $\ge\tau$, then the inclusion is strict.

\paragraph{Setup for loss decomposition.}
Let the online objective split into a ``structural'' term that vanishes on closed, reusable
content and a ``residual'' term that drives adaptation of scaffolds:
$\mathcal{L}_t \;=\; \mathcal{L}_{\text{str}}(\Phi_t) \;+\; \Delta\mathcal{L}_t(\Psi_t \mid \Phi_t),
\text{with}\;\; \mathcal{L}_{\text{str}}(\Phi_t)=0$,
i.e., persistent content acts as a set of invariants on which the model already agrees, so the
remaining improvable part is the residual $\Delta\mathcal{L}_t$ conditioned on $\Phi_t$.
(Equivalently, regard the representation space as a direct sum / quotient
$C_k \cong \underbrace{Z_k/B_k}_{\text{content }\Phi} \oplus \underbrace{\text{residual}}_{\Psi}$,
and view optimization as alternating projection onto these components.)

\paragraph{Lemma 3 (Projection and reuse reduce residual).}
Define the conditional Bayes risk (or population risk) of the residual given the current
persistent set by
$\mathcal{R}(\Phi) \;:=\; \inf_{\Psi}\; \mathbb{E}\big[\,\Delta\mathcal{L}(x;\Psi)\,\big|\,\Phi\big]$.
Two properties hold:
(i) (\emph{Reuse}) Enlarging $\Phi$ augments the invariant subspace / quotient,
so the feasible set for representing data increases; hence $\mathcal{R}(\Phi)$ is monotone nonincreasing
in $\Phi$ (projection theorem / Pythagorean identity for convex losses or least-squares).
(ii) (\emph{Adaptation}) Assumption (iii) ensures the online update is a (stochastic) descent step for
$\Delta\mathcal{L}_t$ at fixed $\Phi_t$, so the realized conditional risk tracks $\mathcal{R}(\Phi_t)$
from above (standard online convex optimization or stochastic approximation).

\paragraph{Descent of conditional residual loss.}
Combining (i) and (ii), taking conditional expectations given $\Phi_{t}$ and $\Phi_{t+1}$,
$\mathbb{E}\big[\Delta\mathcal{L}_{t+1}\,\big|\,\Phi_{t+1}\big]
\;\le\; \mathcal{R}(\Phi_{t+1})
\;\le\; \mathcal{R}(\Phi_{t})
\;\le\; \mathbb{E}\big[\Delta\mathcal{L}_{t}\,\big|\,\Phi_{t}\big]$,
yielding the claimed inequality. If $\Phi_{t+1}\supsetneq \Phi_t$ via a new
nontrivial cycle that explains nonzero residual variance (i.e., reduces $\mathcal{R}$ strictly),
then the middle inequality is strict, hence the overall descent is strict.

\paragraph{Conclusions.}
(1) From Lemma~2, $\Phi_{t+1}\supseteq \Phi_t$ with strict growth when a new nontrivial
persistent cycle is added. (2) From the descent chain above,
$\mathbb{E}[\Delta\mathcal{L}_{t+1}|\Phi_{t+1}] \le \mathbb{E}[\Delta\mathcal{L}_{t}|\Phi_{t}]$,
with strict inequality when $\Phi$ grows nontrivially.

\paragraph{Implications for SbS.}
(a) \emph{Structure-before-specificity.} As $\Phi_t$ monotonically grows, more variance is
explained by persistent content, so the optimization burden shifts from residual adaptation
$\Psi_t$ to reuse of $\Phi_t$ (amortized closure). (b) \emph{Semantics-before-syntax.}
Let $\pi:\mathcal{Z}\to \mathcal{Z}/\!\sim$ be the semantic quotient that identifies latent
states within the same persistent class. Since losses and predictions condition on $\Phi_t$
only through class membership, online encodings asymptotically factor through $\pi$
up to the persistence tolerance $\tau$; i.e., for almost all inputs,
$f_\psi \approx \tilde f \circ \pi$ with $\tilde f$ defined on equivalence classes in $\Phi_t$.

These establish that MAI implements SbS via amortized closure: structure accumulates
monotonically, and residual specificity contracts in expectation under reuse and projection.
\end{proof}

\section{Neural Mechanisms of Temporal Scaffolding and Cycle Closure}
\label{sec:A1}

The macro-scale examples of locomotor gait and spatial homing illustrate how
cycles encode order invariance in behavior. At the neural scale, analogous
mechanisms provide temporal scaffolding for folding timelines into cyclic coordinates and enforce closure on streams of spikes and synaptic inputs. We will separately discuss the neural mechanisms of oscillatory phase coding as temporal scaffolding and coincidence detection for cycle closure next.

\paragraph{Oscillatory phase coding for temporal scaffolding.}
Oscillations supply the contextual scaffold $\Psi$ that folds timelines into cyclic coordinates; coincidence and plasticity then implement boundary cancellation in synaptic space.
What persists are cycles, phase-locked traversals whose endpoints identify on $S^1$, while unmatched fragments dissipate.
This sets up the formal lemma below, which recasts phase-binned spiking as a chain whose boundary vanishes after a full cycle. In the hippocampal-entorhinal system, theta-gamma nesting provides a temporal
scaffold that converts ordered spike trains into phase-locked packets. Within
a theta cycle, gamma bursts encode discrete items whose internal ordering is
irrelevant so long as they fall within the same phase window. Misaligned spikes
cancel as open boundaries, while phase-locked packets close into invariant
loops that can be replayed across cycles \cite{buzsaki2002theta, lisman2005theta}.
Thus, oscillatory phase coding realizes order invariance by mapping temporal
sequences onto cyclic phase structure.

\begin{example}[Temporal scaffolding via hippocampal theta-gamma coding]
The hippocampal-entorhinal system offers a canonical neural example of temporal scaffolding. Slow theta oscillations ($4$--$8$ Hz) furnish the contextual scaffold $\Psi$ by segmenting continuous experience into discrete temporal cycles. Within each cycle, this framework implements the following invariants of neural dynamics:
1) \textbf{Binding} is achieved via faster gamma bursts ($30$--$100$ Hz) encoding content $\Phi$, where neurons representing features of a single event are grouped by firing in the same narrow theta phase window.
2) \textbf{Ordering} is realized through \emph{phase precession}, where the sequence of positions an animal traverses is encoded as a sequence of relative phase offsets within the theta cycle.
3) \textbf{Closure} is enforced by the completion of each theta wave, which resets the system and ensures each experiential trajectory is packaged into a discrete cycle rather than an unbounded chain.
This biological mechanism is a physical realization of the topological identity $\partial^2=0$: the boundary of one sequence (the end of a theta cycle) becomes the beginning of the next, forcing all representations to close. Thus, oscillatory coding transforms linear experience into stable, predictive cycles that form the basis of episodic memory.
\end{example}


\begin{figure}[h]
\centering
\resizebox{\textwidth}{!}{
\begin{tikzpicture}[
  >=Latex, font=\small,
  box/.style={rounded corners, draw, align=center, inner sep=4pt, fill=black!2},
  lbl/.style={inner sep=1pt, font=\scriptsize},
  flow/.style={-Latex, line width=0.6pt},
  dot/.style={circle, fill=black, inner sep=0.9pt},
  axis/.style={line width=0.4pt},
  thickline/.style={line width=0.9pt},
  dashedline/.style={line width=0.6pt, dashed}
]

\node[box, minimum width=5.4cm, minimum height=3.2cm] (P1) at (0,0) {Coincidence detection (time)};
\draw[axis] ($(P1.west)+(0.4,0.6)$) -- ++(4.5,0) node[below right] {$t$};
\foreach \y in {1.2,1.0,0.8} {
  \draw ($(P1.west)+(0.6,\y)$) -- ++(4.1,0);
}
\foreach \x/\y in {0.9/1.2, 1.05/1.0, 1.12/0.8, 2.3/1.2, 2.55/1.0, 2.62/0.8, 3.7/1.2, 3.95/1.0} {
  \draw[thickline] ($(P1.west)+(0.6+\x,\y-0.2)$) -- ++(0,0.4);
}
\draw[dashedline] ($(P1.west)+(1.45,0.7)$) rectangle ++(0.5,0.9);
\draw[dashedline] ($(P1.west)+(2.85,0.7)$) rectangle ++(0.5,0.9);
\node[lbl] at ($(P1.west)+(1.70,1.80)$) {coincidence windows};
\node[lbl, align=left] at ($(P1.south)+(0,0.5)$) {\( \Rightarrow \) temporal closure: \\ aligned spikes form a closed event};

\draw[flow] ($(P1.east)+(0,0)$) -- ++(1.7,0) node[midway, above] {phase coding};

\node[box, minimum width=4.6cm, minimum height=3.2cm, right=1.6cm of P1] (P2) {Phase-locking on \(S^1\)};
\coordinate (C) at ($(P2.center)+(0,0.1)$);
\draw[thickline] (C) circle (1.0);
\foreach \ang in {20,70,140,210,290} {
  \coordinate (p\ang) at ($(C)+({1.0*cos(\ang)},{1.0*sin(\ang)})$);
  \draw[dot] (p\ang) circle (0pt);
}
\node[lbl] at ($(P2.north)+(0,0.2)$) {spike phases on \(\;S^1\)};
\node[lbl, align=center] at ($(P2.south)+(0,0.4)$) {fold time \(\to\) phase (theta/gamma) \\ closure in phase windows};

\draw[flow] ($(P2.east)+(0,0)$) -- ++(1.7,0) node[midway, above] {lift to space};

\node[box, minimum width=7.2cm, minimum height=3.2cm, right=1.6cm of P2] (P3) {Spatial loop \([\gamma]\) and persistence};
\draw ($(P3.west)+(0.4,0.4)$) rectangle ++(6.2,2.3);
\node[lbl] at ($(P3.west)+(3.5,2.9)$) {latent space \(\mathcal{Z}\)};
\coordinate (L1) at ($(P3.west)+(1.1,1.0)$);
\coordinate (L2) at ($(P3.west)+(2.6,2.3)$);
\coordinate (L3) at ($(P3.west)+(4.9,1.9)$);
\coordinate (L4) at ($(P3.west)+(5.2,0.8)$);
\coordinate (L5) at ($(P3.west)+(2.6,0.6)$);
\draw[thickline] (L1) .. controls ($(L1)!0.5!(L2)$) and ($(L2)!0.5!(L3)$) .. (L3)
                 .. controls ($(L3)!0.5!(L4)$) and ($(L4)!0.5!(L5)$) .. (L5)
                 .. controls ($(L5)!0.5!(L1)$) and ($(L1)!0.5!(L2)$) .. (L1);
\draw[dashedline] ($(L1)+(0.12,0.08)$) .. controls ($(L2)+(0.1,0)$) and ($(L3)+(0,-0.05)$) .. ($(L3)+(0.08,-0.05)$)
                 .. controls ($(L4)+(0,0)$) and ($(L5)+(0.05,0.08)$) .. ($(L5)+(0.05,0.08)$)
                 .. controls ($(L1)+(0,0)$) and ($(L2)+(-0.05,-0.05)$) .. ($(L1)+(0.12,0.08)$);
\node[lbl] at ($(P3.west)+(3.5,0.3)$) {homology class \( [\gamma] \in H_1(\mathcal{Z})\)};
\node[box, minimum width=2.4cm] (Pers) at ($(P3.east)+(-1.2,1.7)$) {$\mathrm{Pers}_\tau$};
\draw[flow] ($(P3.east)+(-3.2,1.7)$) -- (Pers.west) node[midway, above, lbl]{barcode};
\node[lbl, align=center] at ($(P3.south)+(0,0.55)$) {persistent loop survives jitter; \\ short-lived chains pruned};

\node[lbl, align=center] at ($(P1.south)!0.5!(P3.south)+(0,-0.2)$)
{time $\Rightarrow$ phase on \(S^1\) $\Rightarrow$ spatial cycle \([\gamma]\) $\Rightarrow$ persistence};

\end{tikzpicture}
}
\caption{From temporal coincidence to spatial topology: spike coincidence closes events in time; phase-locking folds time onto \(S^1\), creating a cyclic temporal scaffold; lifting into latent space \(\mathcal{Z}\) yields a spatial loop whose persistent homology class \([\gamma]\) is invariant under small perturbations and reorderings.}
\label{fig:phase_to_loop}
\end{figure}

\paragraph{Coincidence detection for cycle closure.}
The mechanism of \textbf{cycle closure}, which is essential for forming \textbf{invariant} neural representations, is realized through the synergistic action of \textbf{coincidence detection (CD)} \cite{abeles1982role, konig1996integrator} and \textbf{spike-timing-dependent plasticity (STDP)} \cite{caporale2008spike}. In this process, a designated neuron functions as a coincidence detector, a high-fidelity temporal gate that fires only upon receiving a near-simultaneous volley of presynaptic spikes. STDP acts as the causal learning rule that sculpts the circuit's connectivity: it selectively potentiates synapses from neurons that fire reliably in a predictive sequence just prior to the detector's activation, thereby reinforcing the successful pathway. Concurrently, it depresses or prunes connections from non-causal or temporally uncorrelated neurons, filtering out noise. Over time, this dual mechanism carves out a specific, efficient spatio-temporal pattern. The firing of the terminal detector neuron signifies the successful completion, or closure, of this learned cycle. This definitive, all-or-nothing closure event is what transforms a continuous and noisy stream of sensory data into a discrete, robust, and stable representation that is invariant to minor perturbations that fail to trigger the complete sequence.

Figure~\ref{fig:phase_to_loop} illustrates how local temporal mechanisms give rise
to global spatial invariants. At the neuronal scale, coincidence detection prunes
out-of-phase spikes, enforcing closure within narrow temporal windows. Phase-locking
then folds linear time into a cyclic scaffold ($S^1$), so that events are
represented not by absolute timing but by their position in a recurrent phase
cycle (e.g., theta–gamma coupling). When such phase-coded events are lifted into
the latent manifold $\mathcal{Z}$, trajectories that repeatedly align in phase
become closed spatial loops. Persistent homology extracts these loops as classes
$[\gamma]\in H_1(\mathcal{Z})$, guaranteeing that they remain invariant under
jitter and reordering. Thus, temporal coincidence detection and oscillatory
phase-locking transform the flux of spikes into stable topological invariants of
space, grounding the cycle-closure view of consciousness in a concrete
space–time mechanism.

\begin{example}[Cycle closure via coincidence detection]
At the synaptic level, coincidence detection by neurons provides a clear 
biological instantiation of the cycle-closure principle. A postsynaptic neuron 
typically requires the near-simultaneous arrival of excitatory postsynaptic 
potentials (EPSPs) from multiple presynaptic sources in order to fire. If 
presynaptic spikes arrive asynchronously, their individual contributions decay 
before they can summate; the transient inputs cancel, leaving no lasting trace. 
In homological terms, these are \emph{open chains}: boundary fragments that fail 
to close and thus vanish under the cancellation law $\partial^2=0$.

When inputs coincide within a narrow temporal integration window, however, their 
potentials summate to reach threshold, generating a postsynaptic spike. This 
spike represents the completion of a closed loop of activity: presynaptic 
signals converge, align, and jointly propagate forward. The alignment in time 
acts as the \emph{closure operator}: only when multiple boundary fragments 
synchronize do they form a coherent cycle that persists to the next layer. 
Invariants thus emerge not from isolated fragments but from closed 
constellations of inputs. 

From this perspective, coincidence detection implements a biological test for 
topological closure. Misaligned or weakly correlated events are erased as 
open boundaries, while recurrently co-active inputs are bound into cycles that 
propagate as stable carriers of information. This mechanism explains how neural 
circuits can transform noisy, temporally disordered input streams into 
order-invariant representations: only those patterns that \emph{close in time} 
survive to form the content backbone $\Phi$, while nonclosing fragments remain 
contextual scaffolds $\Psi$ that dissipate.
\end{example}

Together, oscillatory phase coding and coincidence detection instantiate the
closure law $\partial^2=0$ in neural dynamics. They filter away transient
order-dependent fragments, while stabilizing those variations that close into
loops. In this way, the abstract principle of cycle closure is realized
physiologically, linking the topological foundation of invariance to the
biophysical implementation of memory and prediction.


\section{Simulation Blueprint: Testing SbS + CCUP + MAI in a Spiking Neural Network}
\label{sec:snn_blueprint}

We outline a concrete spiking–neural–network (SNN) protocol \cite{wu2025practical} that operationalizes
\textbf{Structure-before-Specificity (SbS)}, the \textbf{Context–Content Uncertainty Principle (CCUP)},
and \textbf{Memory–Amortized Inference (MAI)}. The design links our theory to measurable
closure, persistence, order invariance, and amortized reuse.

\paragraph{Task Suite and Datasets}
\textbf{T1: Looped navigation (latent ring).} An agent receives velocity/odometry and place cues
to traverse closed paths (circles/figure-8). Episodes are short sequences with randomized
micro-event orderings that preserve the same loop class.
\textbf{T2: Cross-modal remapping (visual $\leftrightarrow$ proprioceptive).} The same loop class is
presented via different modalities (pixels vs.\ joint angles), enabling spatial bootstrapping.
\textbf{T3: Social cue alignment.} Two agents (teacher $\to$ student) share loop classes with idiosyncratic
encodings; the student must align to the teacher’s symbols/pulses (social bootstrapping).
Synthetic generators suffice; real sequences can be substituted.

\paragraph{Network Architecture}
\textbf{Populations.} (i) Sensory SNN (Poisson/inhom.\ Poisson input layers), 
(ii) Recurrent latent SNN $\mathcal{Z}$ (excitatory/inhibitory, E/I balance),
(iii) Readout/actor head (linear or spiking readout).
\textbf{Dynamics.} LIF or AdEx neurons with conductance synapses; global $6$–$10$ Hz ``theta''
oscillation gating $30$–$80$ Hz ``gamma'' assemblies (optional but recommended for timescale nesting).
\textbf{Learning.} 
(i) Fast \emph{residual} path $\Psi$: local STDP/eligibility + neuromodulatory error (three-factor rule). 
(ii) Slow \emph{structure} path $\Phi$: surrogate-gradient consolidation (e.g., RSNN backprop-through-time) or synaptic tagging \& capture.
\textbf{Context operators.} Trial-wise gain fields, attention-like modulators, and task-set currents;
implemented so that the induced pushforward $U_\Psi$ is boundary-respecting (no topological rewiring).

\paragraph{Topological Instruments (Closure \& Persistence)}
\textbf{Latent graph.} From spikes in $\mathcal{Z}$, build a time-ordered state graph:
nodes are binned population vectors; edges connect successive bins; optional kNN edges in rate space.  
\textbf{Filtration.} Vietoris–Rips (on rate embeddings) or witness complex (on graph landmarks),
with scale parameter $\epsilon$; produce barcode $\mathcal{B}$ using PH (persistent homology).
\textbf{Cycle set.} $\Phi := \mathrm{Pers}_\tau(H_\bullet(\mathcal{Z}))$: classes with lifetime $\ge\tau$.
\textbf{Closure metric.} Residual boundary norm $R := \|\partial \text{(episode chain)}\|$ estimated
on the state graph (lower is better).  

\paragraph{Objectives and Losses (CCUP \& MAI)}
\textbf{Task loss.} $\mathcal{L}_{\text{task}}$ (trajectory prediction, control reward, or decoding).
\textbf{Residual loss.} $\Delta\mathcal{L}_t := \mathcal{L}_{\text{task}}$ restricted to fast degrees
of freedom $\Psi$ with $\Phi$ fixed (cf.\ Thm.~\ref{thm:MAI_SbS}).  
\textbf{Closure regularizer.} $\mathcal{L}_{\text{clo}} := \lambda_R\, R + \lambda_P\, \sum_{[\gamma]\in \Phi}\!\!\text{penalty}([\gamma])$
where the penalty is small/zero for long-lived cycles and large for short-lived/open chains.  
\textbf{Stability regularizer (SbS precondition).} Lipschitz/spectral norm constraints and
contrastive invariance across order-preserving permutations:
$\mathcal{L}_{\text{stab}} := \mathrm{BNK}(D(f_\psi(x),f_\psi(x^\pi)))$ (bottleneck distance).

\textbf{Total objective.}
\[
\min_{\psi,\theta}\; \underbrace{\mathcal{L}_{\text{task}}}_{\text{performance}}
\;+\; \underbrace{\mathcal{L}_{\text{clo}}}_{\text{closure/persistence}}
\;+\; \underbrace{\mathcal{L}_{\text{stab}}}_{\text{SbS stability}}, 
\qquad
\text{with fast inner updates on }\Psi\text{ only (MAI)}.
\]

\paragraph{Training Protocol (Temporal/Spatial/Social Bootstrapping)}
\textbf{Epoch loop.} For each episode:
(i) \emph{Structure-first retrieval} (reuse $\Phi$): select top-$k$ latent cycles (DTW/graph-alignment) to bias readout.  
(ii) \emph{Residual-only adaptation} (fast STDP/eligibility on $\Psi$) with $\Phi$ frozen.  
(iii) \emph{Closure test via PH}: re-encode post-adaptation spikes, update filtration, compute $\mathcal{B}$, threshold by $\tau$, form $\mathcal{N}$ (new cycles) and $\mathcal{F}$ (now-trivial cycles).  
(iv) \emph{Slow consolidation} (surrogate-gradients/regularization) to fold stable corrections into $\Phi$; update memory $\mathcal{M}$.  
\textbf{Spatial bootstrapping.} Alternate modalities (T2) and reuse the same $\Phi$ across encoders.  
\textbf{Social bootstrapping.} For T3, inject teacher pulses/symbols; student aligns via shared $\Phi$ and residual $\Psi$.

\paragraph{Evaluation Metrics and Hypotheses}
\textbf{H1 (SbS monotonicity).} $|\Phi|$ grows (nondecreasing across epochs); strict increases when novel durable loops appear.  
\textbf{H2 (Residual contraction).} $\mathbb{E}[\Delta\mathcal{L}_{t+1}|\Phi_{t+1}] \le \mathbb{E}[\Delta\mathcal{L}_t|\Phi_t]$.  
\textbf{H3 (Order invariance).} Reordering micro-events within a loop leaves readouts invariant:
$R(s_T)$ unchanged within confidence bounds; BNK distance between latent diagrams $\le \epsilon$.  
\textbf{H4 (Integration).} Sheaf-compatibility proxy: overlaps of retrieved cycles yield consistent readouts
(low disagreement on shared bins); fraction of episodes producing a \emph{single} coherent decode rises.  
\textbf{H5 (Generalization by amortization).} With structure-first retrieval, fewer inner-loop steps achieve target performance; faster adaptation on novel but homologous loops.

\paragraph{Ablations (Causal Tests)}
\textbf{A1 Break closure.} Remove recurrent feedback or theta gating $\Rightarrow$ cycles vanish (barcode collapses),
order invariance fails, residuals increase.  
\textbf{A2 Disable persistence.} Set $\tau\!=\!0$ (keep all fragments) $\Rightarrow$ $|\Phi|$ explodes,
generalization and stability degrade.  
\textbf{A3 No structure-first retrieval.} Start from scratch per episode $\Rightarrow$ slower adaptation, higher $\Delta\mathcal{L}$.  
\textbf{A4 Scramble order.} Randomize within-episode order without preserving loop class $\Rightarrow$ homology changes, performance drops (OI boundary condition violated).  
\textbf{A5 Freeze $\Psi$.} Prevent residual adaptation $\Rightarrow$ short-term errors persist despite existing $\Phi$.

\paragraph{Implementation Notes}
\textbf{Tooling.} Brian2/Norse (PyTorch SNN) for dynamics; Ripser/Gudhi for PH; NetworkX/igraph for state graphs.
\textbf{Parameters (typical).} $N_E=1000$, $N_I=250$, $\text{conn}\approx10\%$, $\tau_{\mathrm{m}}=20$\,ms, 
STDP window $\pm 20$\,ms, neuromodulatory scalar $\in[0,1]$, retrieval $k\in[3,10]$, PH dim $k\in\{1,2\}$,
lifespan threshold $\tau$ from elbow of barcode.
\textbf{Data transforms.} Phase-coded inputs for theta/gamma; additive noise and small temporal jitter for persistence-stability tests.

\begin{center}
\textbf{Pseudocode (One Episode)}    
\end{center}

\begin{algorithm}[H]
\caption{One SNN episode with MAI-as-SbS}
\KwIn{$x_{1:T}$, SNN params $(\psi,\theta)$, cycle memory $\Phi$, threshold $\tau$}
$z_{1:T}\!\leftarrow\! \textsf{SNNEncode}(x_{1:T};\psi)$; 
$\mathcal{C}_{1:T}\!\leftarrow\! \textsf{RetrieveCycles}(z_{1:T},\Phi,k)$ \tcp*{S1}
\For{$t=1$ \KwTo $T$}{
  $\hat{y}_t \leftarrow \textsf{Decode}(z_t,\mathcal{C}_t;\theta)$; 
  $r_t \leftarrow \textsf{Residual}(x_t,\hat{y}_t)$; 
  $(\psi,\theta)\leftarrow \textsf{FastAdapt}_\Psi(\psi,\theta;r_t)$ \tcp*{S2}
}
$\tilde z_{1:T}\!\leftarrow\!\textsf{SNNEncode}(x_{1:T};\psi)$; 
$H\!\leftarrow\!\textsf{PersistentHomology}(\tilde z_{1:T}\cup\Phi)$ \tcp*{S3}
$\mathcal{N}\!\leftarrow\!\{[\gamma]\in H:\mathrm{life}([\gamma])\ge\tau\wedge[\gamma]\notin\Phi\}$; 
$\mathcal{F}\!\leftarrow\!\{[\gamma]\in \Phi:\mathrm{life}([\gamma]|H)<\tau\}$; 
$\Phi\!\leftarrow\!(\Phi\cup\mathcal{N})\setminus\mathcal{F}$; \tcp*{update cycles}
$(\psi,\theta)\leftarrow \textsf{SlowConsolidate}(\psi,\theta;\Phi)$ \tcp*{S4}
\end{algorithm}


\paragraph{Outcome.}
Successful runs should show: rising $|\Phi|$ (with pruning), falling $\Delta\mathcal{L}$, invariance to order-preserving permutations, and improved sample efficiency via structure-first retrieval—empirically validating SbS, CCUP, and MAI as a unified mechanism of closure-induced, persistence-stabilized inference in spiking networks.

\end{document}